\documentclass[twoside,11pt]{article}

\usepackage[abbrvbib, preprint]{jmlr2e}
\usepackage{amsmath,booktabs,array,algorithm,algorithmic,color,enumitem}
\usepackage{flafter}
\hypersetup{hypertexnames=false}

\newcommand{\method}{\mbox{Muon-C}}
\newcommand{\adamw}{AdamW}
\newcommand{\globalmuon}{global unfolded Muon}

\newcommand{\R}{\mathbb{R}}
\newcommand{\C}{\mathbb{C}}
\newcommand{\cF}{\mathcal{F}}

\newcommand{\inner}[2]{\left\langle #1,#2\right\rangle}
\newcommand{\opnorm}[1]{\left\lVert #1\right\rVert_{\mathrm{op}}}
\newcommand{\convnorm}[1]{\left\lVert #1\right\rVert_{\mathrm{conv}}}
\newcommand{\samplenorm}[1]{\left\lVert #1\right\rVert_{\mathrm{samp}}}
\newcommand{\patchnorm}[1]{\left\lVert #1\right\rVert_{\mathrm{patch}}}
\newcommand{\spnorm}[1]{\left\lVert #1\right\rVert_{\mathrm{sp}}}

\ShortHeadings{Operator-Aligned Muon for Convolutional Kernels}{Qing and Li}
\firstpageno{1}

\begin{document}
\title{Muon-C: Operator-Aligned Muon for Convolutional Kernels}
\author{\name Jiaxin Qing \email jxqing@berkeley.edu \\
\addr University of California, Berkeley\\
Berkeley, CA 94720-1776, USA
\AND
\name Lexin Li \email lexinli@berkeley.edu \\
\addr University of California, Berkeley\\
Berkeley, CA 94720-1776, USA}
\maketitle

\begin{abstract}
Muon replaces matrix momentum with an approximately orthogonal polar direction, but its geometry depends on the matrix representation. For convolution, standard unfolding describes a local patch map rather than the convolution operator. We introduce Muon-C, an operator-aligned optimizer that represents kernel momentum as frequency-wise channel-transfer matrices, polarizes these blocks independently, and uses a critical Fourier grid to return updates exactly to the original finite kernel support. We show that the new geometry arises from combining the block partition and Fourier coordinates. The exact-polar direction is a linear minimization oracle under the critically sampled convolution norm. Its worst-case guarantee relative to the continuous convolution-operator norm is never weaker than unfolding and is strictly stronger for $3\times3$ kernels. On CIFAR-10 flow matching with matched applied-update RMS, Muon-C reaches 9.87 FID at 40k iterations, compared with 22.26 for unfolded Muon and 51.31 for Adam. It reaches their final quality using $0.62\times$ and $0.64\times$ their model FLOPs, respectively. Under equal tuning budgets, Muon-C achieves 3.42 FID. Gains persist across data scales and transfer to classification across convolutional architectures.
\end{abstract}

\begin{keywords}
Muon, Optimization, Convolutional kernels, Flow matching, Fourier analysis.
\end{keywords}

%%%%%%%%%%%%%%%%%%%%%%%%%%%%%%%%%%%%%%%%%%%%%%%%%%%%%%%%%%%%%%%%%%%%%%%%%%%%%%%%%%%%%%%
\section{Introduction}

Since its introduction in late 2024 \citep{jordan2024muon}, Muon has rapidly emerged as an important alternative to coordinate-wise optimizers such as Adam \citep{kingma2015adam}. Rather than adapting individual coordinates, Muon reads the momentum of a weight as a matrix and replaces it with an approximately orthogonal polar direction. Its practical relevance is already visible at scale. After reducing training time in language-model speedruns, Muon was used to train a 16B mixture-of-experts model that matched an AdamW run at roughly half the compute \citep{liu2025muonscalable}. More fundamentally, recent analyses show that the exact polar update minimizes the linearized loss over a spectral-norm ball \citep{bernstein2024oldnorm,pethick2025normlmo,chen2025muonconstraints}. Follow-up work further studies the role of this spectral geometry \citep{huang2026spectra,shumaylov2026muonnot}. This interpretation makes a central fact explicit: Muon's geometry depends on which quantities are treated as the rows and columns of the matrix it orthogonalizes.

This dependence, however, creates a fundamental ambiguity for convolutional kernels. A kernel $W\in\R^{C_{\mathrm{out}}\times C_{\mathrm{in}}\times k_h\times k_w}$ has no canonical matrix representation, and the standard practice is to flatten it into a $C_{\mathrm{out}}\times(C_{\mathrm{in}}k_hk_w)$ matrix before applying Muon. The unfolding preserves every coefficient, but it is not neutral bookkeeping. It makes Muon act on a local patch-to-output map and thereby selects a particular optimizer geometry. Resolving this ambiguity matters because convolution remains a central component of modern learning systems, including the U-Nets used in diffusion and flow matching \citep{ronneberger2015unet,lipman2023flow}, ResNets \citep{he2016deep}, and ConvNeXt-style models \citep{liu2022convnet,woo2023convnextv2}. More broadly, how Muon should act on higher-order parameters has only recently become an explicit optimizer-design question \citep{bogachev2026tensorion}.

We take an operator-aligned view of this problem. A convolutional kernel parameterizes a translation-equivariant linear operator. Translation equivariance identifies Fourier modes as the natural invariant subspaces of this operator. In the Fourier basis, convolution is block diagonal, with one $C_{\mathrm{out}}\times C_{\mathrm{in}}$ channel-transfer matrix at each spatial frequency. These transfer matrices are the linear maps implemented by convolution on its translation modes. This observation leads to our central design principle: For a structured parameter, Muon should act on matrices that reflect the natural linear maps of the underlying operator rather than on an arbitrary tensor unfolding. 

Adopting this principle, in this article, we introduce Muon-C, an operator-aligned extension of Muon for convolutional kernels. Muon-C involves three key components. The first is Fourier/operator alignment. Muon-C transforms kernel momentum along its spatial axes into frequency-indexed channel maps, so the matrices seen by Muon correspond directly to the maps implemented by convolution on individual translation modes. The resulting representation is independent of feature-map resolution and follows from the operator's symmetry. The second is independent blockwise polarization. Muon-C applies the polar operation separately to each frequency-wise channel map, instead of applying one global polar operation to the transformed kernel. This distinction is essential because a Fourier change of basis alone does not define a new optimizer. The polar map is unitarily equivariant, so transforming the unfolded momentum, applying one global polar operation, and transforming back returns exactly the ordinary unfolded-Muon direction. The new geometry therefore comes from changing the independently constrained block partition, not from Fourier coordinates by themselves. The third is critical frequency sampling. Muon-C uses exactly a $k_h\times k_w$ Fourier grid, with one frequency sample per stored spatial offset, to make independent block updates compatible with finite kernel support. The dimension-matched DFT is unitary and bijective over the stored coefficients. It permits independent frequency-wise polarization and maps every resulting update exactly back to the original $k_h\times k_w$ support without cropping. The critical grid is therefore the smallest frequency representation that remains bijective over every stored coefficient matrix, and its size is independent of feature-map resolution. Figure~\ref{fig:two-obstacles} summarizes why the latter two components are essential.

\begin{figure}[t!]
\centering
\includegraphics[width=0.75\linewidth,height=2.65in]{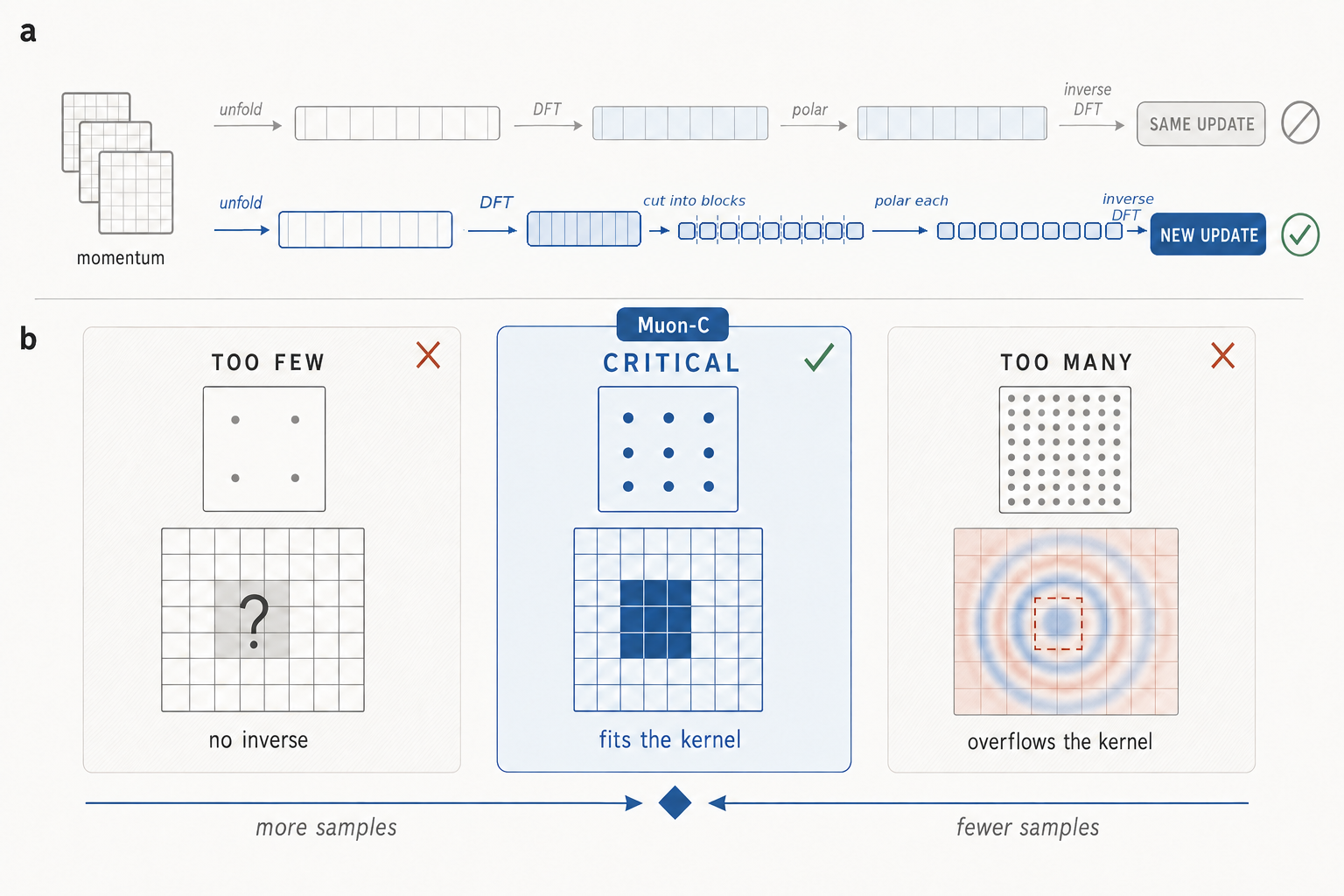}
\caption{Two ingredients are essential for frequency-domain Muon. Independent blockwise polarization changes Muon’s geometry beyond a mere Fourier change of coordinates, whereas critical frequency sampling ensures that the resulting update maps exactly back to the original finite kernel support.}
\label{fig:two-obstacles}
\vspace{-0.3in}
\end{figure}

Theoretically, we provide a rigorous characterization of Muon-C and its relationship to the convolution operator. First, we show that unitary row and column transformations cannot alter the global Muon direction, which isolates the block partition as the essential optimizer-design choice. Second, we show that the critical Fourier grid supplies a unitary, bijective, and support-preserving representation under which Muon-C is the exact linear minimization oracle for the critically sampled convolution norm. Finally, relative to the continuous convolution-operator norm, we establish a worst-case oracle guarantee for Muon-C, such that it is never weaker than standard unfolding for any finite kernel size, and is strictly stronger for the common $3 \times 3$ case, with a guaranteed fraction of the continuous-oracle optimum improved from $1/3$ for unfolding to $9/25$ for Muon-C. 

Empirically, we conduct controlled experiments across generative and discriminative learning tasks. On CIFAR-10 flow matching with matched applied-update RMS, Muon-C reaches 9.87 FID at 40k iterations, compared with 22.26 for unfolded Muon and 51.31 for Adam. It reaches the baselines' 400k-step final quality using only $0.62\times$ and $0.64\times$ their model FLOPs, corresponding to savings of 38\% and 36\%. Under equal optimizer-specific sweep budgets, Muon-C obtains 3.42 FID on CIFAR-10, compared with 3.54 for unfolded Muon and 3.83 for Adam. On ImageNet-1k-32, it improves both the 40k-step FID and the tuned final FID, reaching 21.14 and 12.02 compared with 28.02 and 13.24 for unfolded Muon. On ImageNet-100, Muon-C attains the highest validation accuracy across ResNet-34, ResNet-50, and ConvNeXtV2-T despite their widely different exposure to the convolutional optimizer route. A matched comparison with Muon-S, based on the spatial-block Conv2D duality map of \citet{bernstein2025modular}, shows a stronger early trajectory for translation-frequency blocks than for spatial-offset blocks.

Our contributions are fourfold.
\begin{itemize}[topsep=3pt]
\item We establish operator-aligned matrix geometry as a design principle for extending Muon to convolution. Translation equivariance identifies frequency-wise channel-transfer matrices as the natural linear maps, while independent blockwise polarization defines the new geometry.

\item We develop Muon-C, a finite-support realization of this operator-aligned geometry. It combines Fourier/operator alignment and independent frequency-wise polarization with critical $k_h\times k_w$ sampling, enabling blockwise Muon updates that map exactly back to the original kernel support.

\item We provide an operator-level theoretical characterization of Muon-C, showing that it is the exact linear minimization oracle under the critically sampled convolution norm and enjoys a stronger worst-case guarantee relative to standard unfolding.

\item We demonstrate that Muon-C consistently improves optimization efficiency across generative and classification settings, reaching comparable loss or sample quality substantially earlier than unfolded Muon and Adam/AdamW. A matched spatial-block ablation identifies the contribution of frequency organization.
\end{itemize}

The rest of the article is organized as follows. Section \ref{sec:related} reviews related work, and Section \ref{sec:prelim} introduces notation and the spectral-norm interpretation of Muon. Section \ref{sec:background} develops the operator-aligned view of convolution and motivates frequency-wise block geometry. Section \ref{sec:geometry-practice} presents Muon-C, including the critical-grid construction, theoretical guarantees, and practical implementation. Section \ref{sec:experiments} reports the empirical evaluations, along with robustness studies and mechanism ablations. Section \ref{sec:discussion} concludes with a discussion of future directions. The Appendix collects additional results and proofs.

%%%%%%%%%%%%%%%%%%%%%%%%%%%%%%%%%%%%%%%%%%%%%%%%%%%%%%%%%%%%%%%%%%%%%%%%%%%%%%%%%%%%%%%
\section{Related Work}
\label{sec:related}

\emph{Muon and norm-aware updates.} Muon applies an approximate polar map to matrix momentum, and recent analyses connect this operation to spectral-norm steepest descent and norm-constrained linear minimization \citep{jordan2024muon,bernstein2024oldnorm,pethick2025normlmo,chen2025muonconstraints}. These views make the matrix shape part of the optimizer because the chosen rows, columns, and block boundaries specify the spectral trust region. Adam/AdamW instead adapt coordinates individually and therefore do not require an explicit matrix geometry \citep{kingma2015adam,loshchilov2019decoupled}.

\emph{Modular duality for convolution.} \citet{bernstein2025modular} derive layerwise duality maps from norms chosen to reflect layer semantics and compose them across neural architectures. Their Conv2D map polarizes each spatial-offset channel matrix independently. Muon-S adopts this spatial-block geometry under our matched-update protocol. Muon-C instead polarizes translation-frequency channel maps, with critical sampling preserving finite kernel support.

\emph{Tensor and structured optimization.} K-FAC and Shampoo exploit Kronecker or tensor-mode structure to construct preconditioners \citep{martens2015optimizing,gupta2018shampoo}. Concurrent Tensorion constructs tensor-aware norm oracles through adaptively selected matrix unfoldings \citep{bogachev2026tensorion}. Muon-C takes a different, operator-aligned approach. Rather than selecting a matrix representation from the tensor structure alone, it derives the matrices seen by Muon from the linear operator parameterized by the kernel. For convolution, translation equivariance identifies the direct-sum Fourier representation, whose blocks are frequency-indexed channel-transfer matrices. The two approaches therefore encode different notions of structure. Tensorion adapts Muon to the tensor organization of the parameters, whereas Muon-C aligns Muon with the operator those parameters implement.

\emph{Spectral structure of convolution.} Fourier analysis has long been used to compute, constrain, or bound singular values of convolutional layers \citep{miyato2018spectral,sedghi2019singular,singla2021fantastic}. We use the same transfer-matrix structure to design an optimizer update. Unlike prior uses of convolution spectra as constraints or diagnostics, Muon-C uses the transfer blocks to define the matrices that receive independent optimizer updates. Critical sampling provides a dimension-matched representation of these frequency-wise channel maps for a finite stored kernel.

%%%%%%%%%%%%%%%%%%%%%%%%%%%%%%%%%%%%%%%%%%%%%%%%%%%%%%%%%%%%%%%%%%%%%%%%%%%%%%%%%%%%%%%
\section{Muon as a Spectral-Norm Oracle}
\label{sec:prelim}

We begin with the interpretation of Muon that motivates our operator-aligned construction. The key observation is that Muon's polar update is the exact linear minimization oracle under a spectral-norm constraint. This interpretation makes the matrix representation part of the optimizer geometry, because the matrices, and more generally the independently constrained matrix blocks, on which Muon acts determine the geometry of its update.

We adopt the following notation throughout. For a matrix \(A\), let \(\opnorm{A}\) denote its spectral norm, \(A^*\) its conjugate transpose, and \(\langle A,B\rangle_F=\mathrm{tr}(A^*B)\) the Frobenius inner product; we take its real part in optimization objectives. For real-valued matrices, \(A^*\) reduces to the transpose. For a convolutional kernel, write \([k]=\{0,\ldots,k-1\}\), \(\Omega_k=[k_h]\times[k_w]\), and \(n=k_hk_w\). We represent the kernel as $W=\{W_{uv}\}_{(u,v) \in \Omega_k}$, with $W_{uv}\in\mathbb R^{C_{\mathrm{out}}\times C_{\mathrm{in}}}$, or equivalently \(W\in\mathbb R^{C_{\mathrm{out}}\times C_{\mathrm{in}}\times k_h\times k_w}\).

Muon replaces the momentum of a matrix-shaped weight by its polar factor. For \(M\in\mathbb C^{m\times d}\) with compact singular value decomposition \(M=P\Sigma Q^*\), define the hard-polar map $\operatorname{polar}(M)=PQ^*$. Given gradient \(G^t\), momentum coefficient \(\beta\), learning rate \(\eta\), and a shape-dependent scale \(s\), a Muon step takes the form
\[
M^t=\beta M^{t-1}+(1-\beta)G^t,
\qquad
X^{t+1}=X^t - \eta s\,\operatorname{polar}(M^t).
\]
In practice, computing an exact singular value decomposition at every iteration is expensive, so the polar factor is approximated using a small number of Newton-Schulz iterations on the normalized momentum \citep{jordan2024muon,higham2008functions}.

The polar direction has a precise optimization interpretation. Given a descent direction budget measured in the spectral norm, one solution of $\min_{\|U\|_{\mathrm{op}}\le \rho} \langle M,U\rangle_F$ is $U^\star=-\rho \operatorname{polar}(M)$, with the usual nonuniqueness when \(M\) is rank deficient. Thus, Muon can be viewed as a linear minimization oracle, or equivalently, a steepest linearized descent direction, under a spectral-norm trust region, rather than merely as a heuristic normalization of the momentum \citep{bernstein2024oldnorm,pethick2025normlmo,chen2025muonconstraints}.

This interpretation makes the matrix representation an intrinsic part of the optimizer. The constraint set \(\{U:\|U\|_{\mathrm{op}}\le\rho\}\) depends on which quantities are assigned to the rows and columns of the matrix. More generally, if a parameter is represented by multiple independently constrained matrix blocks, their partition determines a product of spectral-norm balls and hence a different linear minimization oracle. Two invertible representations of the same coefficients can therefore induce different Muon updates even though they parameterize the same underlying object. For an ordinary matrix, the representation is usually given. For a structured parameter such as a convolutional kernel, however, which matrices Muon should act on becomes an optimizer-design choice. Section~\ref{sec:background} develops this choice from the translation-equivariant operator implemented by convolution.

%%%%%%%%%%%%%%%%%%%%%%%%%%%%%%%%%%%%%%%%%%%%%%%%%%%%%%%%%%%%%%%%%%%%%%%%%%%%%%%%%%%%%%%
\section{From Muon to Operator-Aligned Geometry for Convolution}
\label{sec:background}

We now develop the operator-aligned matrix geometry for convolutional kernels. The key question is which matrices should define Muon's spectral geometry for a convolutional operator. We proceed in three steps. We first show that standard unfolding corresponds to a local patch-to-output map, then use translation equivariance to identify frequency-wise channel-transfer matrices as the natural linear maps implemented by convolution, and finally show that these maps induce a genuinely different Muon geometry only when they are polarized as independent blocks.

%%%%%%%%%%%%%%%%%%%%%%%%%%%%%%%%%%%%%%%%%%%%%%%%%%%%%%%%%%%%%%%%%%%%%%%%%%%%%%%%%%%%%%%
\subsection{Muon Geometry Depends on Matrix Representation}

The standard representation concatenates the spatial coefficient matrices into
\[
B_W=
\begin{bmatrix}
W_{00} & W_{01} & \cdots & W_{k_h-1,k_w-1}
\end{bmatrix}
\in\R^{C_{\mathrm{out}}\times k_hk_wC_{\mathrm{in}}}.
\]
This unfolding represents the linear map from a vectorized local input patch to one output vector. Applying Muon to \(B_W\) therefore defines an exact spectral-norm oracle for this local patch-to-output geometry. However, this representation, by choosing the rows and columns on which the polar update acts, selects a particular spectral-norm geometry for the kernel.

The convolution kernel also parameterizes a larger translation-equivariant linear operator over the entire feature map. On an \(H \times W\) grid, forming this operator explicitly would produce a \(C_{\mathrm{out}}HW\times C_{\mathrm{in}}HW\) matrix, which is resolution dependent and impractical to polarize during training. Its translation symmetry, however, decomposes this large operator into small, resolution-independent channel maps. These maps provide the operator-aligned alternative to the local patch geometry, as we develop next.

%%%%%%%%%%%%%%%%%%%%%%%%%%%%%%%%%%%%%%%%%%%%%%%%%%%%%%%%%%%%%%%%%%%%%%%%%%%%%%%%%%%%%%%
\subsection{Translation Equivariance Identifies Fourier Channel Maps}
\label{sec:why-fourier}

Translation equivariance identifies the Fourier basis as the natural representation of the convolution operator. Because convolution commutes with every spatial shift, it preserves the joint eigenspaces of the shift operators, which are precisely the Fourier modes. The layer operator is therefore block diagonal in this basis, with one $C_{\mathrm{out}}\times C_{\mathrm{in}}$ channel map at each spatial frequency. Figure~\ref{fig:why-fourier} illustrates this decomposition for a circular layer with $C_{\mathrm{in}}=C_{\mathrm{out}}=3$, a $3\times3$ kernel, and a $6\times6$ feature grid.

\begin{figure}[t!]
\centering
\includegraphics[width=0.95\linewidth,height=3.5in]{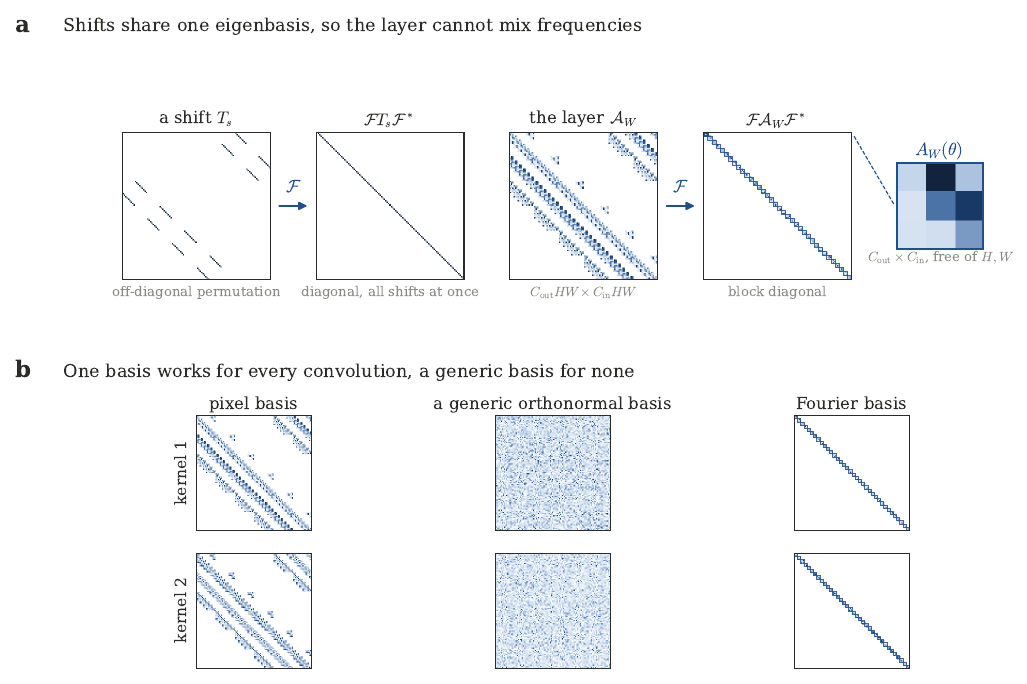}
\caption{Translation equivariance identifies the Fourier channel maps. (a) The Fourier transform simultaneously diagonalizes spatial shifts, decomposing convolution into resolution-independent frequency-wise $C_{\mathrm{out}}\times C_{\mathrm{in}}$ channel-transfer blocks. (b) The same Fourier basis block diagonalizes different convolutional kernels, whereas a generic orthonormal basis does not.}
\label{fig:why-fourier}
\vspace{-0.2in}
\end{figure}

Let $x$ be an input feature map on the infinite lattice $\mathbb{Z}^2$. Define the convolutional layer $\mathcal{A}_W$ by
\[
(\mathcal{A}_Wx)(p)=\sum_{(u,v)\in\Omega_k}W_{uv}\,x\bigl(p-(u,v)\bigr),
\]
and let $T_s$ be the shift $(T_sx)(p)=x(p-s)$. Direct substitution gives $\mathcal{A}_WT_s=T_s\mathcal{A}_W$ for every $s\in\mathbb{Z}^2$. Thus convolution cannot mix the distinct joint eigenspaces of the shift family.

For a fixed frequency $\theta\in[0,2\pi)^2$, the complex exponential $e_\theta(p)=e^{\mathrm{i}\inner{p}{\theta}}$ satisfies $T_se_\theta=e^{-\mathrm{i}\inner{s}{\theta}}e_\theta$. Each $e_\theta$ is a simultaneous eigenvector of every shift, and different frequencies carry different eigenvalue patterns. Hence $\mathcal{A}_W$ acts independently within each frequency.

Only the channel index remains within a frequency. Feeding $x(p)=e_\theta(p)c$ for $c\in\C^{C_{\mathrm{in}}}$ into the layer gives
\begin{equation} \label{eq:transfer-matrix}
\mathcal{A}_W\bigl(e_\theta c\bigr) = e_\theta\,A_W(\theta)c, \qquad 
A_W(\theta)=\sum_{(u,v)\in\Omega_k}W_{uv}\,e^{-\mathrm{i}(u\theta_1+v\theta_2)}.
\end{equation}
The $C_{\mathrm{out}}\times C_{\mathrm{in}}$ matrix $A_W(\theta)$ is the channel-transfer matrix at frequency $\theta$. These matrices are determined entirely by the stored kernel coefficients and are independent of feature-map resolution \citep{sedghi2019singular,singla2021fantastic}. They therefore provide a resolution-independent representation of the linear maps implemented by convolution on its invariant translation modes.

This decomposition has two consequences for the choice of Muon geometry. First, the Fourier basis simultaneously block-diagonalizes every translation-equivariant convolution, whereas a generic orthonormal basis does not preserve this common block structure. Figure~\ref{fig:why-fourier}(b) illustrates this distinction for two kernels. Second, and more importantly, each $A_W(\theta)$ is the channel map actually implemented by convolution on a translation mode. These invariant channel maps provide an operator-aligned choice on which to define Muon's spectral geometry.

%%%%%%%%%%%%%%%%%%%%%%%%%%%%%%%%%%%%%%%%%%%%%%%%%%%%%%%%%%%%%%%%%%%%%%%%%%%%%%%%%%%%%%%
\subsection{Changing Basis Is Not Enough: Blocks Define the Geometry}
\label{sec:basis-change}

Section~\ref{sec:why-fourier} identifies the frequency-wise channel-transfer matrices as the natural operator-aligned maps for Muon. However, expressing the kernel in the Fourier basis alone does not define a new Muon geometry. If the transformed momentum is treated as a single matrix and polarized globally, the resulting direction is exactly the ordinary unfolded-Muon direction, as shown by the next result.

\begin{proposition}[Unitary equivariance of the polar map]
\label{prop:polar-unitary-equivariance}
Let $X\in\C^{m\times d}$, and let $L\in\C^{m\times m}$ and $R\in\C^{d\times d}$ be unitary. For the canonical partial polar factor,
\[
\operatorname{polar}(LXR) = L\,\operatorname{polar}(X)\,R.
\]
Consequently, $L^*\operatorname{polar}(LXR)R^*=\operatorname{polar}(X)$.
\end{proposition}

For a convolutional kernel, concatenating the $n=k_hk_w$ spatial coefficient matrices into the local patch map yields
\[
B_M=
\begin{bmatrix}
M_{00} & M_{01} & \cdots & M_{k_h-1,k_w-1}
\end{bmatrix}
\in\C^{C_{\mathrm{out}}\times nC_{\mathrm{in}}}.
\]
Mixing the spatial offsets by an orthonormal DFT is the particular right-unitary transformation \(B_M\mapsto B_MR_{\mathrm{DFT}}\). Proposition~\ref{prop:polar-unitary-equivariance} implies that applying a single polar map and then undoing this transformation returns \(\operatorname{polar}(B_M)\). Including the descent sign gives \(-\operatorname{polar}(B_M)\). Thus, this unitary row/column transformation cannot by itself produce a new global Muon geometry. What changes the geometry is the partition into independently polarized blocks.

To see how a distinct Muon geometry arises, we now shift attention from the coordinate representation to the partition into independently polarized blocks. Let $\mathcal{R}$ map a kernel to a collection of matrices, and let $\mathcal{R}^{-1}$ return those matrices to parameter space. Then
\[
\mathcal{U}_{\mathcal{R}}(M) = -\mathcal{R}^{-1}\!\left( \left\{\operatorname{polar}\bigl([\mathcal{R}(M)]_b\bigr)\right\}_{b} \right).
\]
For the left/right unitary representations considered here, including the spatial DFT acting on the spatial-offset block columns, \(\mathcal{R}\) preserves the coefficient-space inner product and the global polar direction. The essential choice is therefore not this particular coordinate transformation but the partition of \(\mathcal{R}(M)\) into independently polarized blocks.

\begin{figure}[t!]
\centering
\includegraphics[width=0.3\linewidth,height=1.5in]{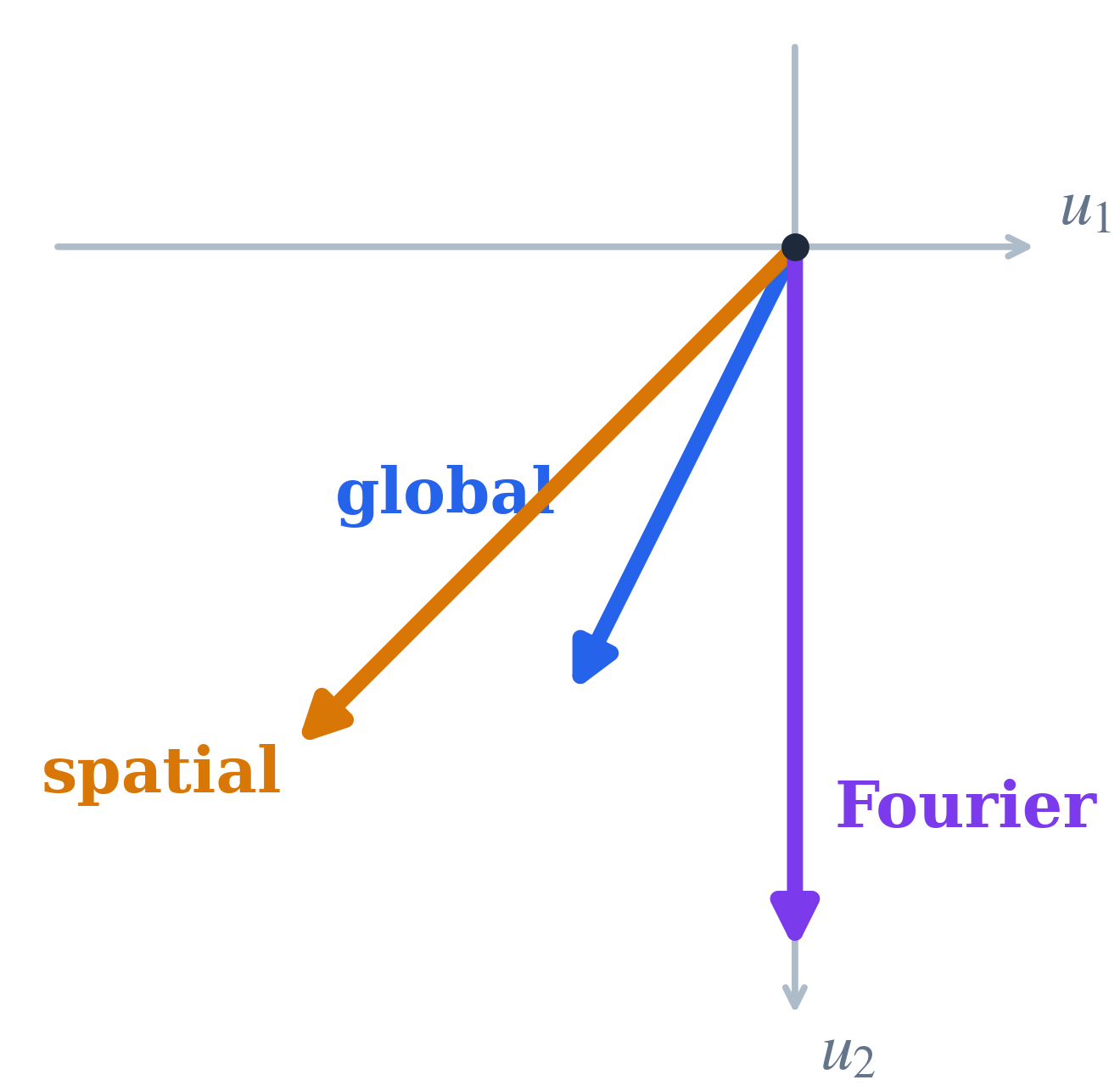}
\caption{Different block partitions induce different Muon directions. For the two-coefficient momentum \(M=(1,2)\), global unfolding, spatial partitioning, and critical-Fourier partitioning yield the directions \(-(1,2)/\sqrt{5}\), \(-(1,1)\), and \((0,-\sqrt{2})\), respectively. The three directions are pairwise non-collinear, illustrating that changing the independently polarized block partition changes the Muon geometry even when the underlying kernel coefficients are the same.}
\label{fig:three-muon-directions}
\vspace{-0.2in}
\end{figure}

The partition changes the geometry because it changes the constraint set. A single block imposes one spectral-norm constraint $\opnorm{B_U}\leq\rho$, whereas a partition into multiple blocks imposes one constraint per block, producing a product of spectral-norm balls. These constraint sets define genuinely different linear minimization oracles even when the underlying coordinate transformation is unitary. Blockwise polarization also provides a useful interpretation of this geometry. Each polarized block is a partial isometry, so the blocks contribute at a common spectral scale rather than retaining the relative magnitudes present in the momentum. Our proposed Muon-C therefore extends Muon's within-matrix spectral equalization to the operator-aligned decomposition across translation modes.

We focus on three block partitions of the same convolutional kernel. The first is global unfolding, which treats $B_M$ as a single block and is the exact linear minimization oracle for the local patch norm $\patchnorm{U} = \opnorm{B_U}$. The second is critical-Fourier partitioning, used by Muon-C, which independently polarizes one \(C_{\mathrm{out}}\times C_{\mathrm{in}}\) channel map at each critical translation frequency, thereby aligning the blocks with the invariant maps of the convolution operator. The third is spatial partitioning, used by Muon-S, which independently polarizes one \(C_{\mathrm{out}}\times C_{\mathrm{in}}\) channel matrix at each stored spatial offset. This is the spatial-block polar geometry of the Conv2D duality map in \citet{bernstein2025modular}, up to layerwise scaling. Muon-S matches Muon-C in block count and block shape while organizing blocks by spatial offset, providing a baseline for isolating the role of Fourier organization. Figure \ref{fig:three-muon-directions} gives a minimal two-coefficient example in which these three partitions produce pairwise non-collinear directions, demonstrating that changing the block partition alone can change the Muon update.

The operator view therefore identifies independent frequency-wise channel maps, rather than Fourier coordinates alone, as the appropriate Muon geometry for convolution. One challenge remains. A finite \(k_h\times k_w\) kernel has only \(k_hk_w\) free coefficient matrices, whereas its transfer function is defined over a continuum of frequencies. Section \ref{sec:geometry-practice} resolves this finite-support mismatch through critical frequency sampling, completing the Muon-C construction by combining Fourier/operator alignment, independent frequency-wise polarization, and exact preservation of the original kernel support.

%%%%%%%%%%%%%%%%%%%%%%%%%%%%%%%%%%%%%%%%%%%%%%%%%%%%%%%%%%%%%%%%%%%%%%%%%%%%%%%%%%%%%%%
\section{Muon-C: Critical Fourier Geometry for Finite Convolutional Kernels}
\label{sec:geometry-practice}

We next develop Muon-C and show how the operator-aligned geometry can be realized for a finite convolutional kernel. We proceed in four steps. We first introduce the critical Fourier grid that makes independent frequency-wise polarization compatible with finite kernel support. We then define the resulting Muon-C update, and establish its exact linear-oracle characterization as well as its guarantee relative to the continuous convolution-operator geometry. Finally, we present the complete algorithm along with practical implementation.

%%%%%%%%%%%%%%%%%%%%%%%%%%%%%%%%%%%%%%%%%%%%%%%%%%%%%%%%%%%%%%%%%%%%%%%%%%%%%%%%%%%%%%%
\subsection{From Continuous Operator Geometry to Critical Sampling}
\label{sec:natural-geometry}

We first formalize the natural operator geometry that Muon-C seeks to approximate. For a kernel \(U\), the transfer matrix \(A_U(\theta)\) in \eqref{eq:transfer-matrix} describes the channel map implemented by convolution at frequency \(\theta\). This motivates the convolution norm
\[
\convnorm{U} = \sup_{\theta\in[0,2\pi)^2} \opnorm{A_U(\theta)} .
\]
This quantity is the induced \(\ell_2\) norm of multichannel convolution on the infinite spatial lattice. It also upper-bounds the corresponding operator norm on any finite zero-padded grid, while for circular convolution it reduces to a maximum over discrete frequencies. Moreover, changing from corner-indexed to centered kernel offsets only multiplies \(A_U(\theta)\) by a unit-modulus scalar and therefore leaves its singular values unchanged. Thus, \(\convnorm{\cdot}\) is an intrinsic, resolution-independent property of the convolutional kernel and provides the operator-level analogue of the matrix spectral norm underlying ordinary Muon.

Combining this operator norm with the finite kernel support gives the natural convolutional oracle
\begin{equation*} \label{eq:natural-finite-lmo}
\min_{U\in\R^{C_{\mathrm{out}}\times C_{\mathrm{in}}\times k_h\times k_w}} \ \inner{M}{U}
\quad \text{subject to} \quad
\convnorm{U}\leq\rho.
\end{equation*}
This is the convolutional counterpart of the spectral-norm oracle that defines Muon in Section \ref{sec:prelim}. An idealized full-grid construction would treat the channel map at every frequency as an independent variable, even though a finite kernel must generate all of these maps from the same \(n=k_hk_w\) stored coefficient matrices. Independently polarizing the full-grid blocks therefore generally yields an inverse transform with nonzero coefficients outside \(\Omega_k\), so the result is not a feasible finite-support update without projection. Appendix~\ref{app:theory} gives the formal full-grid oracle and its proof, while Appendix \ref{app:empirical-diagnostics} measures this obstruction in trained U-Net momenta.

Muon-C resolves this mismatch using the critical \(k_h\times k_w\) Fourier grid. The grid contains exactly \(n=k_hk_w\) frequency samples, matching the number of stored coefficient matrices, and the corresponding spatial DFT is unitary and bijective on those coefficients. Its inverse maps every collection of critical-grid blocks exactly back to the original \(k_h \times k_w\) support, while each block remains an evaluation of the convolutional transfer matrix \(A_U(\theta)\). The critical grid therefore provides an operator-aligned, bijective, and support-preserving representation on which the frequency blocks can be polarized independently. Section \ref{sec:critical} uses this representation to define the Muon-C update.

%%%%%%%%%%%%%%%%%%%%%%%%%%%%%%%%%%%%%%%%%%%%%%%%%%%%%%%%%%%%%%%%%%%%%%%%%%%%%%%%%%%%%%%
\subsection{Muon-C Update}
\label{sec:critical}

We now define the Muon-C geometry on the critical Fourier grid. Let
\[
\theta_p=\frac{2\pi p}{k_h}, \qquad \phi_q=\frac{2\pi q}{k_w}, \qquad p\in[k_h], \; q\in[k_w],
\]
giving \(n=k_hk_w\) critical frequencies in total. The corresponding critically sampled convolution norm is
\begin{equation} \label{eq:sampled-norm}
\samplenorm{U} = \max_{p\in[k_h],q\in[k_w]} \opnorm{A_U(\theta_p,\phi_q)} .
\end{equation}
This norm measures the largest channel-transfer operator norm over the critical grid and defines the trust-region geometry used by Muon-C.

To express this geometry in the coordinates used by the algorithm, let \(\mathcal F_{\Omega_k}\) denote the orthonormal two-dimensional DFT over the \(k_h\times k_w\) kernel grid, and write
\[
\widehat U(p,q) = \bigl[ \mathcal F_{\Omega_k}(U) \bigr](p,q) = \frac{1}{\sqrt n} A_U(\theta_p,\phi_q).
\]
It follows that
\begin{equation} \label{eq:sample-kfb-relation}
\| U \|_{\mathrm{samp}} = \sqrt n \, \max_{p,q} \bigl \| \widehat U(p,q) \bigr \|_{\mathrm{op}} .
\end{equation}
The factor \(\sqrt n\) results from the orthonormal DFT convention and rescales the corresponding trust-region radius without changing its blockwise polar direction. Importantly, up to this common normalization, the matrices \(\widehat U(p,q)\) retain their operator interpretation as evaluations of the same transfer matrix \(A_U(\theta)\) that defines the continuous convolution norm \(\convnorm{U}\). At the same time, the critical DFT is unitary and bijective on the stored kernel coefficients.

In the representation framework of Section \ref{sec:basis-change}, Muon-C therefore takes \(\mathcal{R}=\mathcal F_{\Omega_k}\) and treats each \(C_{\mathrm{out}}\times C_{\mathrm{in}}\) frequency-wise channel map as an independent block. Given kernel momentum \(M^t\), let
\[
\widehat M^t = \mathcal F_{\Omega_k}(M^t).
\]
Muon-C polarizes each critical-frequency block independently,
\[
\widehat O^t(p,q) = \operatorname{polar} \left(\widehat M^t(p,q)\right), \qquad p\in[k_h], \; q\in[k_w],
\]
and maps the resulting blocks back to the kernel domain,
\begin{equation*}
O^t = \mathcal F_{\Omega_k}^{-1} \bigl(\widehat O^t\bigr).
\end{equation*}
Figure \ref{fig:pipeline} illustrates this update pipeline, and contrasts it with global unfolded Muon that applies a single polar map to the local patch matrix.

\begin{figure}[t!]
\centering
\includegraphics[width=\linewidth]{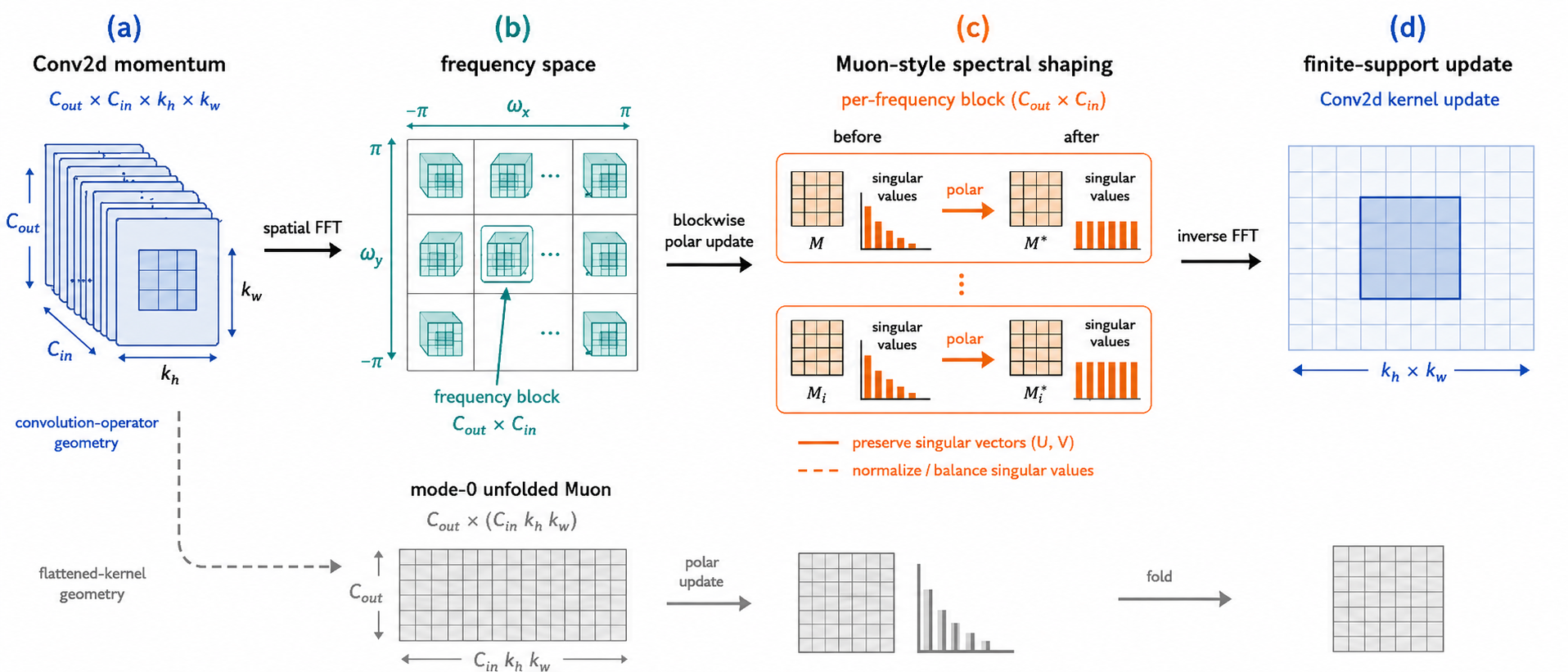}
\caption{Muon-C update on the critical Fourier grid. Muon-C transforms the kernel momentum into frequency-wise channel maps on the critical grid, polarizes these blocks independently, and applies the inverse transform to obtain an update on the original finite kernel support. In contrast, unfolded Muon treats the local patch map as a single block and applies one global polar operation.}
\label{fig:pipeline}
\vspace{-0.2in}
\end{figure}

Because the critical DFT is bijective, \(O^t\) has exactly the same \(k_h\times k_w\) spatial support as the original kernel and requires no cropping or support projection. Moreover, for real-valued momentum, the Fourier blocks are conjugate symmetric, and the canonical polar map preserves this symmetry. Hence the inverse transform returns a real-valued kernel update. Muon-C therefore combines operator-aligned frequency blocks, independent blockwise polarization, and exact preservation of the finite kernel support. Section \ref{sec:theoretical-properties} shows that this direction is the exact linear minimization oracle for the sampled geometry in \eqref{eq:sampled-norm} and quantifies its guarantee relative to the continuous convolution norm.

%%%%%%%%%%%%%%%%%%%%%%%%%%%%%%%%%%%%%%%%%%%%%%%%%%%%%%%%%%%%%%%%%%%%%%%%%%%%%%%%%%%%%%%
\subsection{Theoretical Properties of Muon-C}
\label{sec:theoretical-properties}

We next establish two properties of Muon-C. First, its update is the exact linear minimization oracle under the sampled convolution norm. Second, this sampled norm provides a controlled approximation to the continuous convolution norm, with a worst-case approximation factor no larger than that of global unfolding and strictly smaller for \(3\times3\) kernels.

We begin with the exact oracle characterization. Recall from \eqref{eq:sample-kfb-relation} that, under the orthonormal kernel DFT, $\| U \|_{\mathrm{samp}} = \sqrt{n}\max_{p,q}
\| \widehat U(p,q) \|_{\mathrm{op}}$, with $n=k_hk_w$. The sampled-norm constraint therefore separates across the critical-frequency blocks, yielding the following result.

\begin{proposition}[Exact sampled-convolution oracle]
\label{prop:exact-lmo}
Let \(M\in\mathbb{R}^{C_{\mathrm{out}}\times C_{\mathrm{in}}\times k_h\times k_w}\) be a kernel momentum and let \(\widehat M=\mathcal F_{\Omega_k}(M)\). Then one solution of
\begin{equation*}
\min_{U} \langle M,U\rangle \qquad \text{subject to} \qquad \| U \|_{\mathrm{samp}} \leq \rho
\end{equation*}
is given in the Fourier domain by
\begin{equation}
\label{eq:sample-lmo}
\widehat U_f^\star(p,q) = -\frac{\rho}{\sqrt n} \operatorname{polar}\!\left(\widehat M(p,q)\right), \qquad p\in[k_h], \; q\in[k_w].
\end{equation}
\end{proposition}

Up to the common scale \(\rho/\sqrt n\), Proposition~\ref{prop:exact-lmo} is exactly the Muon-C direction defined in Section~\ref{sec:critical}. The result follows from Parseval's identity and spectral-nuclear norm duality, applied independently to each critical-frequency block. For real-valued kernels, conjugate symmetry of the Fourier blocks is preserved by the canonical polar map, so the resulting spatial update remains real valued.

Having established exactness for the sampled norm, we next investigate how closely this norm controls the continuous convolution norm. To compare the three block partitions introduced in Section~\ref{sec:basis-change}, recall the local patch norm
$
\| U \|_{\mathrm{patch}} = \| B_U \|_{\mathrm{op}},
$
associated with global unfolded Muon, and define the spatial-block norm
\[
\| U \|_{\mathrm{sp}} = \max_{(u,v)\in\Omega_k} \| U_{uv} \|_{\mathrm{op}},
\]
associated with Muon-S. Muon-C is similarly associated with the sampled convolution norm \(\|\cdot\|_{\mathrm{samp}}\) in \eqref{eq:sampled-norm}.

For a square kernel, the Conv2D norm of \citet{bernstein2025modular} satisfies
\[
\mathrm{Conv2D.norm}(U) = n\sqrt{\frac{C_{\rm in}}{C_{\rm out}}}\,\|U\|_{\rm sp}, \qquad n=k^2,
\]
since $\|A\|_{\rm RMS\to RMS}=\sqrt{C_{\rm in}/C_{\rm out}}\,\|A\|_{\rm op}$. Its duality map is $n^{-1}\sqrt{C_{\rm out}/C_{\rm in}}\,\operatorname{polar}(M_{uv})$ at each spatial offset. The kernel-area factor $n$ also appears in the spatial upper bound below.

The remaining issue is how much the transfer matrix can grow between the critical-grid samples. Trigonometric interpolation controls this growth through the Lebesgue constant. Define the one-dimensional cardinal functions
\[
\ell_p^{(k)}(\theta) = \frac{1}{k} \sum_{r=0}^{k-1} e^{-\mathrm{i}r(\theta-2\pi p/k)}, \qquad p\in[k],
\]
and the corresponding Lebesgue constant
\[
\Lambda_k = \sup_{\theta\in[0,2\pi)} \sum_{p=0}^{k-1} \left|\ell_p^{(k)}(\theta)\right|.
\]
The next theorem places all three finite-support geometries on a common scale relative to the continuous convolution norm.

\begin{theorem}[Unified finite-support norm comparison]
\label{thm:unified-norm-comparison}
Let \(n=k_hk_w\). Every kernel \(U\) supported on \(\Omega_k\) satisfies
\begin{equation} \label{eq:norm-hierarchy}
\| U \|_{\mathrm{sp}} \leq \| U \|_{\mathrm{patch}} \leq \| U \|_{\mathrm{samp}} \leq \| U \|_{\mathrm{conv}},
\end{equation}
and
\begin{equation} \label{eq:unified-upper-bounds}
\| U \|_{\mathrm{conv}} \leq \min\left\{ n \| U \|_{\mathrm{sp}}, \sqrt n\, \| U \|_{\mathrm{patch}}, \Lambda_{k_h}\Lambda_{k_w} \| U \|_{\mathrm{samp}} \right\}.
\end{equation}
Moreover,
\begin{equation} \label{eq:lebesgue-sqrt-bound}
\Lambda_k\leq\sqrt{k}, \qquad\text{and hence}\qquad \Lambda_{k_h}\Lambda_{k_w}\leq\sqrt n.
\end{equation}
For a \(3 \times 3\) kernel, \(\Lambda_3=5/3\), so the corresponding worst-case approximation factors for Muon-S, global unfolding, and Muon-C are \(9\), \(3\), and \(25/9\), respectively.
\end{theorem}

Theorem \ref{thm:unified-norm-comparison} quantifies how tightly each tractable geometry controls the continuous convolution norm. Global unfolding controls the local patch map with distortion \(\sqrt n\), whereas Muon-C directly controls the critical-frequency channel maps with distortion \(\Lambda_{k_h}\Lambda_{k_w}\). By \eqref{eq:lebesgue-sqrt-bound}, the Muon-C distortion is no larger than that of global unfolding for any finite kernel size. For the common \(3\times3\) kernel, the inequality is strict:
\[
\underbrace{\frac{25}{9}}_{\text{Muon-C}} < \underbrace{3}_{\text{unfolded Muon}} < \underbrace{9}_{\text{Muon-S}}.
\]
Thus, among the three block geometries we consider, the operator-aligned critical-frequency geometry provides the tightest worst-case control of the continuous convolution norm.

The norm comparison also yields a direct guarantee relative to the ideal continuous convolution oracle. Let \(U_g^\star\), \(U_f^\star\), and \(U_s^\star\) denote exact radius-\(\rho\) minimization oracles for the patch, sampled Fourier, and spatial-block norms, respectively, and define
\[
\kappa_g=\sqrt n, \qquad \kappa_f=\Lambda_{k_h}\Lambda_{k_w}, \qquad \kappa_s=n.
\]

\begin{corollary}[Continuous-convolution oracle guarantee]
\label{cor:unified-lmo}
For \(r\in\{\mathrm{g},\mathrm{f},\mathrm{s}\}\), let \(\widetilde U_r = U_r^\star/\kappa_r\). Then \(\widetilde U_r\) is feasible for the radius-\(\rho\) continuous convolution-norm ball and satisfies
\begin{equation} \label{eq:unified-lmo-factor}
-\langle M,\widetilde U_r \rangle \geq \frac{1}{\kappa_r} \max_{\convnorm{V}\leq\rho} -\langle M,V \rangle.
\end{equation}
\end{corollary}

Corollary~\ref{cor:unified-lmo} gives \(1/\kappa_r\) as the fraction of the optimal continuous-oracle linear objective guaranteed by each surrogate geometry after rescaling to the same convolution-norm ball. For a \(3\times3\) kernel, these factors are
\[
\underbrace{\frac{9}{25}}_{\text{Muon-C}} > \underbrace{\frac{1}{3}}_{\text{unfolded Muon}} > \underbrace{\frac{1}{9}}_{\text{Muon-S}}.
\]
More generally, \(\Lambda_{k_h}\Lambda_{k_w}\leq\sqrt n\) implies that the worst-case continuous-oracle guarantee for Muon-C is at least as strong as that of global unfolding for every finite kernel size. The evaluated classification architectures also contain \(2\times2\), \(4\times4\), and \(7\times7\) kernels. For each of these square kernel sizes, \(\Lambda_k^2\leq k\) preserves the guarantee that the Muon-C distortion is no larger than the global-unfolding distortion. These guarantees concern exact, single-step linear minimization oracles. 
% These guarantees characterize the quality of a single linearized update. They do not by themselves imply faster nonconvex training.  
% Instead, they provide a geometric criterion that we examine empirically in Section \ref{sec:spatial-block-ablation}, by comparing candidate updates under a common convolution-operator budget.

\begin{table}[t!]
\centering
\scriptsize
\begin{tabular}{p{0.15\linewidth}p{0.20\linewidth}p{0.15\linewidth}p{0.12\linewidth}p{0.18\linewidth}}
\toprule
Geometry & Matrices seen by Muon & Trust norm & Distortion & Polar shape of exact LMO \\
\midrule
Global unfolded & One local patch map $B_M$ & $\patchnorm{U}$ & $\sqrt n$ & $\operatorname{polar}(B_M)$ \\
\method{} & Critical-frequency channel maps & $\samplenorm{U}$ & $\Lambda_{k_h}\Lambda_{k_w}$ & $\cF^{-1}\{\operatorname{polar}(\widehat M)\}$ \\
Spatial ablation & One channel map per offset & $\spnorm{U}$ & $n$ & $\{\operatorname{polar}(M_{uv})\}_{u,v}$ \\
\bottomrule
\end{tabular}
\par
\normalsize
\caption{Comparison of three Muon geometries under the continuous convolution norm. Muon-C has a worst-case approximation factor no larger than global unfolding, and strictly smaller for $3\times3$ kernels. Muon-S serves as the matched spatial-block ablation. The last column shows the polar shape of each exact LMO, omitting signs and scales.}
\label{tab:representation-comparison}
\end{table}

Table~\ref{tab:representation-comparison} summarizes the three geometries, their independently polarized blocks, associated trust norms, distortion factors, and the polar shapes of their exact surrogate directions. Together, Proposition~\ref{prop:exact-lmo}, Theorem~\ref{thm:unified-norm-comparison}, and Corollary~\ref{cor:unified-lmo} characterize Muon-C as an exact Muon oracle for a tractable finite-support surrogate of the convolution-operator geometry, with a worst-case norm comparison that is no weaker than global unfolding and is strictly tighter for \(3 \times 3\) kernels. Proofs are given in Appendix~\ref{app:theory}.

Muon-C does incur additional optimizer-side computation relative to global unfolded Muon because it performs a kernel-axis FFT and multiple channel-block polar operations. This overhead is independent of feature-map resolution but is not included in the model-FLOP accounting used in Section \ref{sec:applied-update-rms}. We therefore separately report measured wall-clock and peak-memory overhead in Section \ref{sec:applied-update-rms}.

%%%%%%%%%%%%%%%%%%%%%%%%%%%%%%%%%%%%%%%%%%%%%%%%%%%%%%%%%%%%%%%%%%%%%%%%%%%%%%%%%%%%%%%
\subsection{Practical Muon-C Algorithm}
\label{sec:practical}

Sections \ref{sec:critical} and \ref{sec:theoretical-properties} define the Muon-C direction and characterize it as an exact linear minimization oracle under the sampled convolution norm in~\eqref{eq:sampled-norm}. We now turn this direction into a practical optimizer through four steps: efficient computation of the frequency-wise polar factors, calibration and routing of the resulting parameter updates, treatment of computational cost and common convolution variants under the critical-grid construction, and assembly of these elements into a complete optimization step. Together, these components preserve the operator-aligned geometry developed above while making Muon-C directly applicable to standard convolutional networks.

The first step is the efficient computation of the Muon-C direction. Muon-C requires no additional optimizer state beyond the momentum already maintained by Muon. Given the momentum $M^t$, we apply the $k_h\times k_w$ orthonormal DFT $\mathcal{F}_{\Omega_k}$ along the two spatial kernel axes, producing the critical-frequency blocks
\[
\widehat M^t = \mathcal{F}_{\Omega_k}(M^t).
\]
The polar factor of each $C_{\rm out}\times C_{\rm in}$ block $\widehat M^t(p,q)$ is then approximated using a small number of Newton-Schulz iterations, and the resulting blocks are transformed back with $\mathcal{F}_{\Omega_k}^{-1}$. Thus, the practical computation directly implements the blockwise direction defined in Section \ref{sec:critical}, with the exact polar map replaced by its Newton-Schulz approximation. Importantly, the number of frequency blocks is $k_hk_w$ and therefore depends only on the kernel support, not on the spatial resolution of the feature map. Appendix \ref{app:implementation} gives the batching, dtype, and complex Newton-Schulz details.

The second step is to convert this direction into the actual parameter update. The oracle characterization in Proposition \ref{prop:exact-lmo} determines the Muon-C direction up to a common positive scale. Under a radius-$\rho$ sampled-norm constraint, the exact Fourier-domain oracle carries the common factor $\rho/\sqrt{n}$, where $n=k_hk_w$. In the practical optimizer, we absorb this trust-region radius and DFT normalization into an optimizer scale $s_R$ and calibrate the realized parameter-space update. This changes the magnitude of the step but not the operator-aligned block geometry established in Sections \ref{sec:critical} and \ref{sec:theoretical-properties}.

For a representation $R$, let $U_R(M^t)$ denote its shaped descent direction. The applied parameter update takes the generic form
\[
W^{t+1} = (1-\eta\lambda)W^t + s_R\eta\, U_R(M^t),
\]
where $\eta$ is the base learning rate and $\lambda$ is the decoupled weight decay. For an eligible convolutional kernel with $n=k_hk_w$, we use
\[
s_g = 0.2\sqrt{\max\{C_{\rm out},\,nC_{\rm in}\}}, \qquad s_f = \frac{0.2}{\operatorname{RMS}(O)+\epsilon},
\]
for global unfolded Muon and Muon-C, respectively, where $O$ denotes the unscaled Muon-C direction in the original kernel layout. The global rule is the matched-RMS scaling used in large-scale Muon training, which multiplies the polar direction of an $A\times B$ matrix by $0.2\sqrt{\max\{A,B\}}$ \citep{liu2025muonscalable}. Because Muon-C returns its direction directly in the kernel layout, $s_f$ instead normalizes the realized kernel update. Under exact hard-polar shaping, both rules target directional RMS $0.2$ before multiplication by the base learning rate, placing the two Muon geometries on the same AdamW-referenced scale.

The same practical step also requires a routing rule specifying which parameters use the convolution-specific geometry. Conv2d kernels with spatial area $k_hk_w>1$ use either Muon-C or global unfolded Muon, while $1\times1$ convolutions, biases, normalization parameters, and all remaining weights use AdamW. Section \ref{sec:experiments} evaluates these optimizers both under controlled applied-update scale and under optimizer-specific learning-rate tuning.

The third step is to examine the practical consequence of the critical-grid construction. In addition to preserving the original finite kernel support, critical sampling keeps the computational cost of Muon-C tied to the kernel rather than the activation resolution. A construction based on the full feature grid would allocate one channel-transfer block per feature frequency, whereas Muon-C uses only one block per critical kernel frequency. For instance, for a $3\times3$ kernel with $C_{\rm in}=C_{\rm out}=256$, the frequency blocks occupy approximately $3$ MiB on the critical kernel grid, compared with approximately $12.4$ GiB on a $224\times224$ feature grid, even before accounting for Newton-Schulz temporaries. Critical sampling thus resolves both the finite-support obstruction in Section \ref{sec:natural-geometry} and the resolution-dependent cost of full-grid polarization. Appendix \ref{app:implementation} gives the full memory calculation.

Within this third step, we also consider how broadly the same construction applies across convolutional layers. For grouped convolution, the transfer matrix is block diagonal across groups at every frequency, and $\|U\|_{\rm conv} = \max_g \|U_g\|_{\rm conv}$.  The trust-region constraint and linear objective therefore separate across groups, so applying Muon-C independently within each group gives the corresponding linear minimization oracle for the complete layer. The norm comparison in Theorem \ref{thm:unified-norm-comparison}  and the continuous-oracle guarantee in Corollary \ref{cor:unified-lmo} apply groupwise without modification. Depthwise convolution is the scalar-block special case of the same construction. The same kernel-grid update can also be applied to a strided convolution because its stored $k_h\times k_w$ kernel support is unchanged, although an operator-norm characterization that explicitly incorporates subsampling remains open.

\begin{algorithm}[t!]
\caption{One operator-aligned Muon-C update on a finite convolutional kernel}
\label{alg:muon-c-step}
\small
\begin{algorithmic}[1]
\REQUIRE Weight $W^t$, gradient $G^t$, momentum $M^{t-1}$, learning rate $\eta$, momentum coefficient $\beta$, weight decay $\lambda$, kernel grid $\Omega_k$, numerical stabilizer $\epsilon$
\ENSURE Updated weight $W^{t+1}$
\STATE $M^t \leftarrow \beta M^{t-1} + (1-\beta)G^t$
\STATE $\widetilde M^t \leftarrow \beta M^t + (1-\beta)G^t$ for Nesterov momentum, or $\widetilde M^t \leftarrow M^t$
\STATE $\widehat M^t \leftarrow \mathcal{F}_{\Omega_k}(\widetilde M^t)$ over spatial kernel axes
\FOR{all critical-frequency blocks $(p,q)\in\Omega_k$}
\STATE $\widehat O^t(p,q) \leftarrow \operatorname{polar}_{\mathrm{NS}} \!\left(\widehat M^t(p,q)\right)$
\ENDFOR
\STATE $O^t \leftarrow \mathcal{F}_{\Omega_k}^{-1}(\widehat O^t)$
\STATE $s_f \leftarrow 0.2/(\operatorname{RMS}(O^t)+\epsilon)$
\STATE $W^{t+1} \leftarrow (1-\eta\lambda)W^t-\eta s_f O^t$
\end{algorithmic}
\end{algorithm}

The fourth and final step is to assemble the preceding components into one practical Muon-C update. Algorithm \ref{alg:muon-c-step} combines momentum formation, the critical-grid Fourier transform, independent frequency-wise polar shaping, the inverse transform, and the calibrated parameter update. We distinguish the stored exponential moving-average momentum $M^t$ from the momentum $\widetilde M^t$ that is actually shaped by Muon-C. With Nesterov momentum, the latter includes the additional look-ahead combination used in the implementation.

The resulting optimizer retains the operator-aligned geometry developed in Section \ref{sec:background} and the critical-grid construction of Sections \ref{sec:natural-geometry} to  \ref{sec:theoretical-properties}, while requiring only kernel-sized Fourier transforms and small $C_{\rm out}\times C_{\rm in}$ polar operations.

%%%%%%%%%%%%%%%%%%%%%%%%%%%%%%%%%%%%%%%%%%%%%%%%%%%%%%%%%%%%%%%%%%%%%%%%%%%%%%%%%%%%%%%
\section{Numerical Experiments}
\label{sec:experiments}

We evaluate Muon-C along three complementary dimensions. First, we compare it with global unfolded Muon under controlled applied-update scale to determine whether critical Fourier geometry improves optimization efficiency independently of update magnitude. Second, we test whether the advantage remains robust across stochastic replication, optimizer-specific tuning, larger data scale, different learning objectives, and convolutional architectures. Third, we use a matched spatial-block ablation to isolate operator alignment from blockwise polarization alone.

%%%%%%%%%%%%%%%%%%%%%%%%%%%%%%%%%%%%%%%%%%%%%%%%%%%%%%%%%%%%%%%%%%%%%%%%%%%%%%%%%%%%%%%
\subsection{Experimental Setup}
\label{sec:exp-setup}

We evaluate Muon-C in both generative learning and discriminative learning settings. Table~\ref{tab:evaluation-suite} summarizes the evaluation suite. For generative modeling, we train the same U-Net architecture with flow matching \citep{ronneberger2015unet,lipman2023flow} on CIFAR-10 and ImageNet-1k resized to $32\times32$. Both experiments use FID-50k as the primary quality metric and a 400k-step training budget. For discriminative learning, we train ResNet-34, ResNet-50 \citep{he2016deep}, and ConvNeXtV2-T \citep{woo2023convnextv2} on ImageNet-100 at $224\times224$ resolution and evaluate full-validation top-1 accuracy.

\begin{table}[t!]
\centering
\small
\begin{tabular}{lllll}
\toprule
Dataset & Task & Model & Primary metric & Budget \\
\midrule
CIFAR-10 & Flow matching & U-Net & FID-50k $\downarrow$ & 400k steps \\
ImageNet-1k-32 & Flow matching & U-Net & FID-50k $\downarrow$ & 400k shared steps \\
ImageNet-100 & Classification & ResNet-34 & Top-1 $\uparrow$ & 100 epochs \\
ImageNet-100 & Classification & ResNet-50 & Top-1 $\uparrow$ & 100 epochs \\
ImageNet-100 & Classification & ConvNeXtV2-T & Top-1 $\uparrow$ & 100 epochs \\
\bottomrule
\end{tabular}
\par
\normalsize
\caption{Evaluation suite across generative learning and discriminative learning. ImageNet-1k is resized to $32 \times 32$ for flow matching, whereas ImageNet-100 classification uses $224 \times 224$ inputs.}
\label{tab:evaluation-suite}
\vspace{-0.1in}
\end{table}

\begin{table}[t!]
\centering
\small
\setlength{\tabcolsep}{4pt}
\begin{tabular}{lccccc}
\toprule
Model & Eligible layers & Conv2d layers & Kernel sizes & Eligible params & Share \\
\midrule
U-Net & 52 & 65 & $3\times3$ & 29.9M / 35.7M & $83.8\%$ \\
ResNet-34 & 33 & 36 & $7\times7$, $3\times3$ & 21.1M / 21.3M & $98.9\%$ \\
ResNet-50 & 17 & 53 & $7\times7$, $3\times3$ & 11.3M / 23.7M & $47.8\%$ \\
ConvNeXtV2-T & 22 & 22 & $7\times7$ dw, $4\times4$, $2\times2$ & 1.9M / 27.9M & $6.7\%$ \\
\bottomrule
\end{tabular}
\par
\normalsize
\caption{Coverage of the convolutional optimizer route across architectures. Conv2d weights with $k_h k_w > 1$ are eligible for Muon-C or global unfolded Muon. All remaining parameters use AdamW.}
\label{tab:routing-coverage}
\vspace{-0.1in}
\end{table}

Table~\ref{tab:routing-coverage} reports how much of each architecture is exposed to the convolutional optimizer route. A parameter is eligible when it is a Conv2d weight with spatial area $k_hk_w>1$. Eligible tensors use either Muon-C or global unfolded Muon, while all remaining parameters use AdamW. The U-Net routes 29.9M of its 35.7M parameters through the convolutional geometry, making it close to a direct test of the proposed geometry. ResNet-34 provides an even higher-exposure classification setting at 98.9\%. ResNet-50 is a mixed case at 47.8\% because its bottleneck design places substantial capacity in $1\times1$ projections. ConvNeXtV2-T provides the lowest-exposure test at 6.7\%. Its eligible weights are primarily depthwise spatial kernels, while most parameters lie in pointwise expansion matrices. Together, these architectures test the proposed geometry when it acts on nearly all, roughly half, or only a small minority of model parameters.

All flow-matching comparisons use the same U-Net, data pipeline, batch size, mixed precision, warmup-constant schedule, 400k-step budget, exponential moving average, and parameter-routing rule. Final generative quality is evaluated using 50k generated samples, 100 Dopri5 sampling steps, and fixed real-data statistics. The primary CIFAR-10 comparison matches the applied-update RMS on the eligible convolutional route using the scaling in Section~\ref{sec:practical}. The robustness studies examine repeated runs at a common learning rate, equal-budget optimizer-specific sweeps, scaling to ImageNet-1k-32, and transfer to ImageNet-100 classification. Appendix~\ref{app:implementation} gives the remaining implementation and evaluation details.

%%%%%%%%%%%%%%%%%%%%%%%%%%%%%%%%%%%%%%%%%%%%%%%%%%%%%%%%%%%%%%%%%%%%%%%%%%%%%%%%%%%%%%%
\subsection{Optimization Efficiency under Controlled Update Scale}
\label{sec:applied-update-rms}

\begin{figure}[t!]
\centering
\includegraphics[width=0.65\linewidth,height=2.5in]{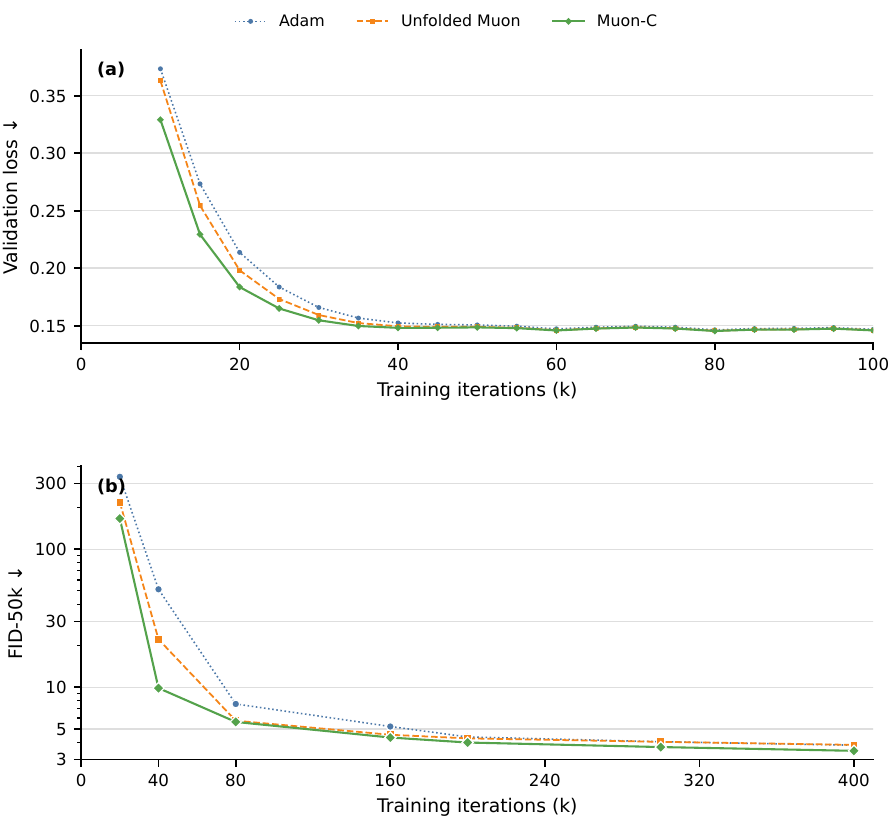}
\caption{CIFAR-10 flow matching under matched applied-update RMS. (a) Validation loss over the first 100k iterations. (b) FID-50k at seven shared checkpoints over the full 400k-step budget. Lower is better.}
\label{fig:matched-update-rms-convergence}
\vspace{-0.2in}
\end{figure}

We first examine whether replacing global patch geometry with critical Fourier geometry improves optimization efficiency. CIFAR-10 provides the primary controlled comparison, where applied-update RMS is explicitly matched between the two Muon geometries. We then test whether the same advantage persists when the training distribution is scaled to ImageNet-1k-32 under a common base learning rate. Finally, we express the CIFAR-10 trajectories in terms of model FLOPs to quantify compute-to-quality efficiency. Figures~\ref{fig:matched-update-rms-convergence}--\ref{fig:matched-rms-fid-flops} report these comparisons.

Figure~\ref{fig:matched-update-rms-convergence} shows the primary controlled comparison on CIFAR-10. At the base learning rate $2\times10^{-4}$, the calibration described in Section~\ref{sec:practical} matches applied-update RMS on the eligible convolutional route. This controls overall parameter-space update scale while changing the matrices on which Muon applies its polar shaping. Muon-C has the lowest validation loss at all 19 evaluations between 10k and 100k iterations, reaching 0.184 at 20k and approaching its plateau near 0.146 by 80k. The separation is more pronounced in sample quality. Muon-C has the lowest FID-50k at all seven shared checkpoints. At 40k iterations, it reaches 9.87 FID, compared with 22.26 for global unfolded Muon and 51.31 for Adam. At the 400k-step endpoint, the corresponding values are 3.47, 3.83, and 3.81. Appendix~\ref{app:empirical-diagnostics} further shows the calibrated Muon-C run has a slightly smaller applied displacement.

\begin{figure}[t!]
\centering
\includegraphics[width=0.65\linewidth,height=2.5in]{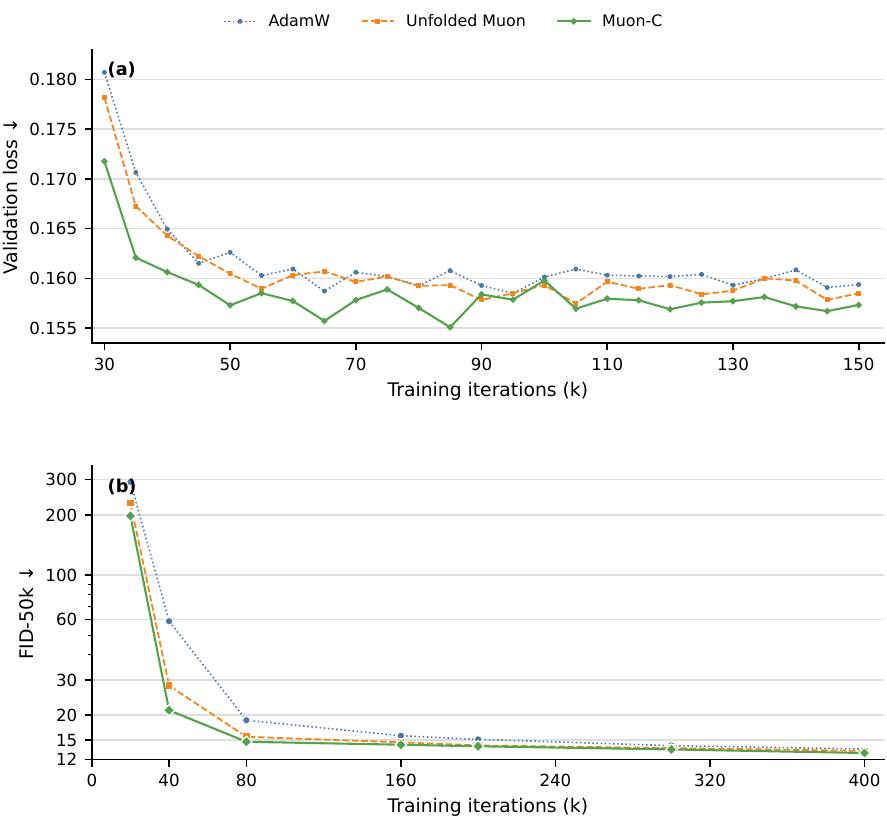}
\caption{ImageNet-1k-32 flow matching under a common base learning rate of $2 \times 10^{-4}$. (a) Validation loss over the first 150k iterations. (b) FID-50k at seven shared checkpoints over the full 400k-step budget. Lower is better.}
\label{fig:imagenet32-convergence}
\vspace{-0.2in}
\end{figure}

Figure~\ref{fig:imagenet32-convergence} tests the same comparison at a substantially larger data scale. The experiment changes the training distribution from CIFAR-10 to ImageNet-1k resized to $32\times32$ while keeping the flow-matching objective, U-Net architecture, parameter routing, and base learning rate $2\times10^{-4}$ fixed. Unlike the primary CIFAR-10 experiment, it uses a common base learning rate rather than explicit RMS calibration. Muon-C has the lowest validation loss at 23 of 25 evaluations between 30k and 150k iterations and the lowest FID-50k at all seven shared checkpoints. Its FID is 21.14 at 40k iterations, compared with 28.02 for global unfolded Muon and 58.88 for AdamW. At 80k iterations, Muon-C reaches 14.71, a quality level that the baselines require approximately 149k and 233k iterations to attain. At 400k iterations, Muon-C reaches 12.93 and remains 0.32 below global unfolded Muon and 0.60 below AdamW. The early and final advantages therefore persist when the training distribution is enlarged by more than an order of magnitude.

\begin{figure}[t!]
\centering
\includegraphics[width=0.6\linewidth,height=2.15in]{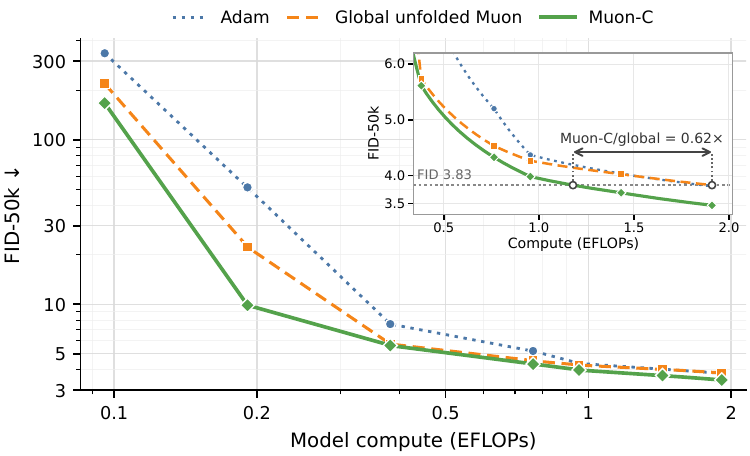}
\caption{CIFAR-10 compute-to-quality efficiency under matched applied-update RMS. FID-50k is plotted against cumulative U-Net forward-backward FLOPs. Lower and further left indicates better sample quality with less model compute.}
\label{fig:matched-rms-fid-flops}
\vspace{-0.2in}
\end{figure}

Figure~\ref{fig:matched-rms-fid-flops} shows that the CIFAR-10 trajectory advantage translates directly into compute-to-quality efficiency. It replots the seven matched-RMS FID checkpoints against cumulative U-Net forward-backward FLOPs, so lower and further left indicates better sample quality at a smaller model-compute budget. The Muon-C curve remains below both baselines throughout the measured range. Log-log interpolation shows that Muon-C reaches the final 400k-step quality of global unfolded Muon and Adam using approximately $0.62\times$ and $0.64\times$ their model FLOPs, corresponding to savings of 38\% and 36\%. Its extra per-step work consists of a kernel-axis transform and $k_hk_w$ polar operations on $C_{\rm out}\times C_{\rm in}$ blocks. As discussed in Section~\ref{sec:practical}, this cost depends on kernel size rather than feature-map resolution.

\begingroup
\begin{table}[t!]
\centering
\setlength{\tabcolsep}{3.5pt}
\resizebox{\linewidth}{!}{
\begin{tabular}{lccccc}
\toprule
Optimizer & Time / step (ms) $\downarrow$ & Relative time & Peak memory (GiB) $\downarrow$ & Steps to FID 4.0 (k) $\downarrow$ & time (h) $\downarrow$ \\
\midrule
Adam & 73.19 & $1.00\times$ & 6.42 & 312.9 & 6.36 \\
Global unfolded Muon & 86.02 & $1.18\times$ & 6.31 & 311.7 & 7.45 \\
\method{} & 99.39 & $1.36\times$ & 6.31 & \textbf{197.5} & \textbf{5.45} \\
\bottomrule
\end{tabular}}
\par
\normalsize
\caption{CIFAR-10 U-Net runtime and time to FID 4.0. Evaluation time is excluded.}
\label{tab:cifar-runtime-efficiency}
\vspace{-0.1in}
\end{table}
\endgroup

Table~\ref{tab:cifar-runtime-efficiency} complements the model-FLOP comparison with runtime. Muon-C takes 35.8\% longer per step than Adam and 15.5\% longer than global unfolded Muon, while requiring substantially fewer steps to reach the common target of FID 4.0. Combining the measured step time with the interpolated target crossing gives a projected training time of 5.45 hours for Muon-C, compared with 6.36 hours for Adam and 7.45 hours for global unfolded Muon, reductions of 14.3\% and 26.8\%, respectively. Its 6.31 GiB peak allocated memory is essentially the same as global unfolded Muon and slightly below Adam in this benchmark.

Together, these comparisons provide the primary evidence for improved optimization efficiency. Muon-C produces a substantially stronger early-training trajectory, and the advantage persists to the final training budget. On CIFAR-10, the controlled-update comparison yields substantial compute-to-quality savings and reduces projected training time to a common FID target. On ImageNet-1k-32, the same ordering persists at a larger data scale under a common learning rate. The empirical ordering is consistent with the tighter worst-case control of the continuous convolution geometry in Corollary~\ref{cor:unified-lmo}. For a $3\times3$ kernel, Muon-C guarantees $9/25$ of the continuous-oracle objective, compared with $1/3$ for global unfolding.

%%%%%%%%%%%%%%%%%%%%%%%%%%%%%%%%%%%%%%%%%%%%%%%%%%%%%%%%%%%%%%%%%%%%%%%%%%%%%%%%%%%%%%%
\subsection{Robustness across Tuning, Data Scale, and Tasks}
\label{sec:robustness}

We next examine the robustness and transfer of the optimization-efficiency gains along three dimensions. We test reproducibility across independent training runs under a common learning rate, compare the optimizers under equal tuning budgets while allowing each to select its own learning rate, and evaluate transfer from flow matching to classification across architectures with substantially different routing coverage. These experiments assess whether the benefit of Muon-C extends beyond a particular random seed, learning-rate choice, data scale, learning objective, or architecture.

%%%%%%%%%%%%%%%%%%%%%%%
\subsubsection{Common-Learning-Rate Replication}

\begin{table}[b!]
\centering
\small
\begin{tabular}{lcc}
\toprule
Optimizer & CIFAR-10 FID $\downarrow$ & ImageNet-1k-32 FID $\downarrow$ \\
\midrule
Adam / \adamw{} & $3.94\pm0.07$ & $13.53\pm0.05$ \\
Global unfolded Muon & $3.85\pm0.08$ & $13.04\pm0.06$ \\
\method{} & $\mathbf{3.60\pm0.07}$ & $\mathbf{12.62\pm0.05}$ \\
\bottomrule
\end{tabular}
\par
\normalsize
\caption{Final flow-matching performance under a common base learning rate of $2 \times 10^{-4}$. Entries report FID-50k. The coordinate-wise baseline is Adam on CIFAR-10 and AdamW on ImageNet-1k-32. Lower is better.}
\label{tab:common-lr-results}
\end{table}

Table~\ref{tab:common-lr-results} shows that the Muon-C advantage is stable across repeated runs under a common base learning rate of $2\times10^{-4}$ and otherwise identical protocols. On CIFAR-10, Muon-C achieves $3.60\pm0.07$ FID, compared with $3.85\pm0.08$ for global unfolded Muon and $3.94\pm0.07$ for Adam. Appendix~\ref{app:empirical-diagnostics} reports the paired results for three training seeds, with Muon-C attaining the best endpoint in every repeat. On ImageNet-1k-32, Muon-C similarly reaches $12.62\pm0.05$ FID, compared with $13.04\pm0.06$ for global unfolded Muon and $13.53\pm0.05$ for AdamW. The ordering is therefore reproducible across stochastic training runs and holds on both flow-matching datasets.

%%%%%%%%%%%%%%%%%%%%%%%
\subsubsection{Equal-Budget Optimizer-Specific Tuning}

\begin{table}[t!]
\centering
\small
\begin{tabular}{lcccc}
\toprule
& \multicolumn{2}{c}{CIFAR-10} & \multicolumn{2}{c}{ImageNet-1k-32} \\
\cmidrule(lr){2-3}\cmidrule(lr){4-5}
Optimizer & Selected LR & FID $\downarrow$ & Selected LR & FID $\downarrow$ \\
\midrule
Adam / \adamw{} & $2\times10^{-4}$ & $3.83\pm0.03$ & $3\times10^{-4}$ & $13.53\pm0.05$ \\
Global unfolded Muon & $8\times10^{-4}$ & $3.54\pm0.04$ & $1\times10^{-3}$ & $13.24\pm0.04$ \\
\method{} & $5\times10^{-4}$ & $\mathbf{3.42\pm0.03}$ & $6\times10^{-4}$ & $\mathbf{12.02\pm0.02}$ \\
\bottomrule
\end{tabular}
\par
\normalsize
\caption{Equal-budget optimizer-specific learning-rate tuning. Each optimizer receives the same ten-run sweep, and the selected learning rate and final FID-50k are reported. The coordinate-wise baseline is Adam on CIFAR-10 and AdamW on ImageNet-1k-32.}
\label{tab:main-results}
\end{table}

Table~\ref{tab:main-results} shows that the advantage remains when every optimizer receives the same tuning budget but may select its own learning rate. For each optimizer, we sweep ten learning rates from $1\times10^{-4}$ to $1\times10^{-3}$ in increments of $1\times10^{-4}$ and select the value with the lowest final FID-50k. Muon-C reaches 3.42 FID on CIFAR-10, compared with 3.54 for global unfolded Muon and 3.83 for Adam. On ImageNet-1k-32, it reaches 12.02, compared with 13.24 for global unfolded Muon and 13.53 for AdamW. Its improvements over global unfolded Muon are therefore 0.12 and 1.22 FID on the two datasets.

The selected learning rates are $5\times10^{-4}$ for Muon-C and $8\times10^{-4}$ for global unfolded Muon on CIFAR-10, and $6\times10^{-4}$ and $1\times10^{-3}$ on ImageNet-1k-32. Appendix~\ref{app:sweep-budget} evaluates the selected runs at $0.1B$ and $0.5B$, where $B=400$k steps. Muon-C has the lowest FID at every reported budget fraction on both datasets, so its advantage appears before the final tuned endpoint.

%%%%%%%%%%%%%%%%%%%%%%%
\subsubsection{Transfer to ImageNet-100 Classification}

\begin{figure}[t!]
\centering
\includegraphics[width=0.65\linewidth,height=2.65in]{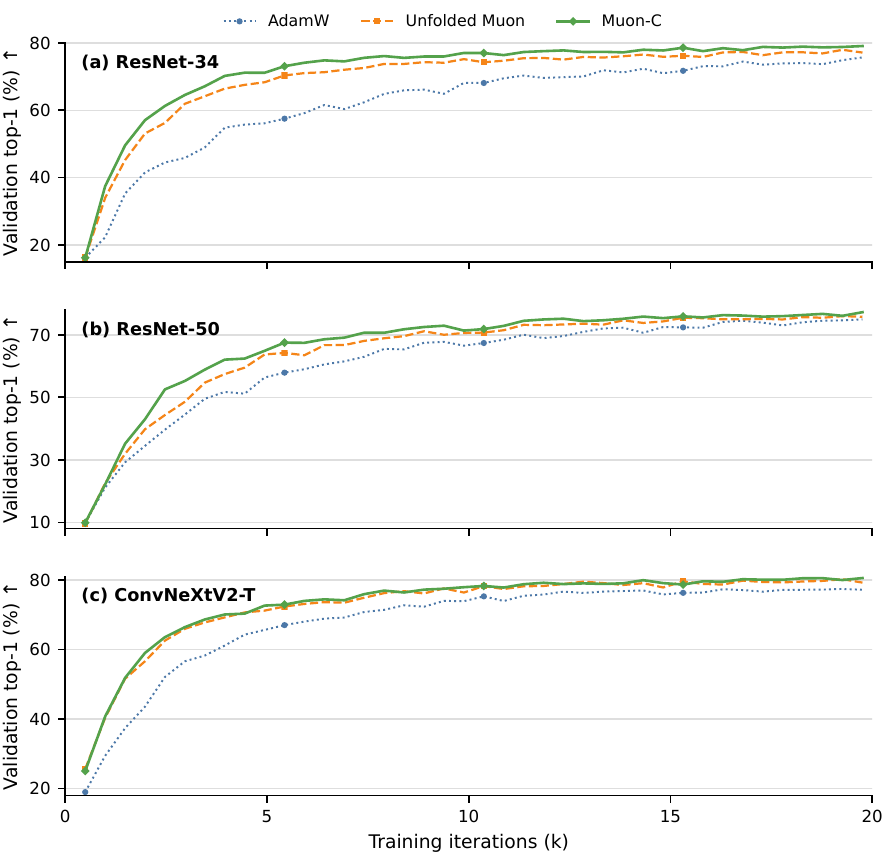}
\caption{Early optimization trajectories on ImageNet-100 classification. Full-validation top-1 accuracy over the first 20k optimizer updates for (a) ResNet-34, (b) ResNet-50, and (c) ConvNeXtV2-T. Higher is better.}
\label{fig:imagenet100-classification-convergence}
\vspace{-0.2in}
\end{figure}

\begin{table}[t!]
\centering
\small
\begin{tabular}{llccc}
\toprule
Setting & Metric & \adamw{} & Unfolded Muon & \method{} \\
\midrule
ResNet-34 & Top-1 $\uparrow$ & $79.74\pm0.16$ & $79.64\pm0.13$ & $\mathbf{80.22\pm0.15}$ \\
ResNet-50 & Top-1 $\uparrow$ & $79.24\pm0.17$ & $79.28\pm0.18$ & $\mathbf{79.88\pm0.16}$ \\
ConvNeXtV2-T & Top-1 $\uparrow$ & $82.86\pm0.09$ & $84.62\pm0.12$ & $\mathbf{84.88\pm0.10}$ \\
\bottomrule
\end{tabular}
\par
\normalsize
\caption{ImageNet-100 full-validation top-1 accuracy after 100 training epochs. Results are reported for architectures with substantially different exposure to the convolutional optimizer route. Higher is better.}
\label{tab:transfer-results}
\end{table}

Figure~\ref{fig:imagenet100-classification-convergence} and Table~\ref{tab:transfer-results} show that the Muon-C advantage transfers from generative flow matching to discriminative classification. We train ResNet-34, ResNet-50, and ConvNeXtV2-T on ImageNet-100 using the same parameter-routing and applied-update-RMS protocol as in the controlled CIFAR-10 comparison. The architectures expose 98.9\%, 47.8\%, and 6.7\% of their parameters to the convolutional optimizer route, respectively. This range provides a controlled test of whether the benefit survives across architectures with very different reliance on spatial convolution.

Figure~\ref{fig:imagenet100-classification-convergence} shows that Muon-C has the highest validation accuracy at every reported evaluation point on all three architectures, with the ordering emerging within the first few thousand updates. At 2k updates, it reaches 57.06\% top-1 accuracy on ResNet-34, compared with 53.10\% for global unfolded Muon and 41.54\% for AdamW. The corresponding values are 43.00\%, 39.80\%, and 34.48\% on ResNet-50, and 59.06\%, 56.66\%, and 43.60\% on ConvNeXtV2-T. The consistent early separation shows the optimization-efficiency advantage is not specific to flow matching.

Table~\ref{tab:transfer-results} shows that the early advantage also persists to the end of training. Muon-C reaches 80.22\% on ResNet-34, improving by 0.58 percentage points over global unfolded Muon and 0.48 points over AdamW. On ResNet-50, it reaches 79.88\%, improving by 0.60 and 0.64 points, respectively. On ConvNeXtV2-T, Muon-C reaches 84.88\%, compared with 84.62\% for global unfolded Muon and 82.86\% for AdamW. The ConvNeXtV2-T result is particularly informative because only 1.9M of its 27.9M parameters use the convolutional route, yet Muon-C still improves over global unfolded Muon. 

Across the three classification architectures, Muon-C shows larger gains early in training and smaller gains at the 100-epoch endpoint. The robustness studies show consistent gains across the evaluated seeds, tuning settings, datasets, and architectures. We next compare frequency-wise and spatial-offset blocks using the matched Muon-S ablation in Section~\ref{sec:spatial-block-ablation}.

%%%%%%%%%%%%%%%%%%%%%%%%%%%%%%%%%%%%%%%%%%%%%%%%%%%%%%%%%%%%%%%%%%%%%%%%%%%%%%%%%%%%%%%
\subsection{Does Operator Alignment Drive the Gain?}
\label{sec:spatial-block-ablation}

\begin{figure}[t!]
\centering
\includegraphics[width=0.65\linewidth,height=2.65in]{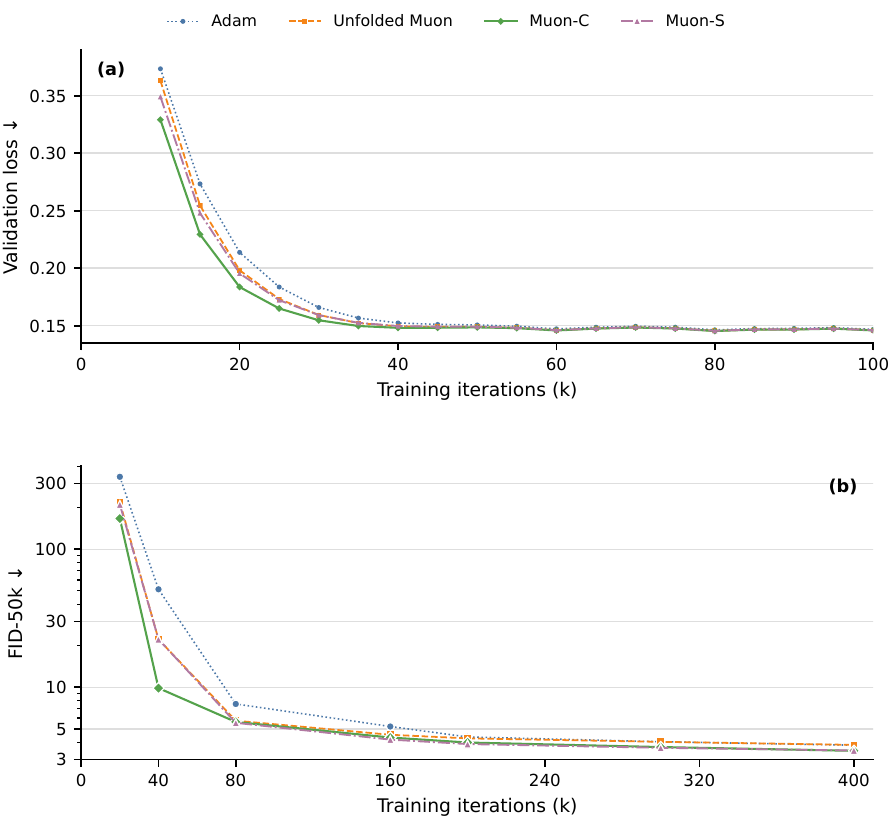}
\caption{CIFAR-10 operator-alignment ablation under matched applied-update RMS. (a) Validation loss over the first 100k iterations. (b) FID-50k at seven shared checkpoints over the full 400k-step budget. Muon-S replaces the kernel DFT with the identity while matching block count, block shape, routing, learning rate, update RMS, and training protocol. Lower is better.}
\label{fig:spatial-block-ablation-convergence}
\vspace{-0.2in}
\end{figure}

The preceding experiments establish that Muon-C improves optimization efficiency across training protocols, data scales, tasks, and architectures. We now isolate the mechanism by separating blockwise channel polarization from operator-aligned translation-frequency organization. Muon-S evaluates the Conv2D spatial-block duality direction of \citet{bernstein2025modular} with our shared momentum, routing, and update-scale settings. Operationally, it replaces the kernel DFT in Muon-C with the identity transform and applies the same channel-matrix polar map independently at each spatial offset. It matches Muon-C in block count, block shape, parameter routing, learning rate, applied-update RMS, and training protocol. Muon-C uses translation-frequency blocks, whereas Muon-S uses spatial-offset blocks.

Figure~\ref{fig:spatial-block-ablation-convergence} shows a clear early optimization advantage from the operator-aligned frequency organization. Muon-C has the lowest validation loss at all 19 evaluations from 10k to 100k iterations. At 20k iterations, its loss is 0.184, compared with 0.195 for Muon-S. The separation is larger in sample quality. At 20k iterations, Muon-C reaches 166.83 FID, compared with 211.25 for Muon-S. At 40k iterations, the values are 9.87 and 22.28. The gap narrows later, but Muon-C remains slightly ahead at the 400k-step endpoint with 3.47 FID compared with 3.49 for Muon-S. Both blockwise variants finish ahead of global unfolded Muon at 3.83 and Adam at 3.81.

Under the matched protocol, frequency-wise blocks yield a stronger early trajectory than spatial-offset blocks. This result supports the role of translation-frequency organization in Muon-C's optimization efficiency.

The ablation comparison also connects directly to the theory. Muon-S and Muon-C are the exact linear minimization oracles for the spatial-block and critically sampled convolution geometries. For a $3\times3$ kernel, Theorem~\ref{thm:unified-norm-comparison} gives worst-case approximation factors of $9$ and $25/9$, while Corollary~\ref{cor:unified-lmo} guarantees $1/9$ and $9/25$ of the continuous-oracle objective.

%%%%%%%%%%%%%%%%%%%%%%%%%%%%%%%%%%%%%%%%%%%%%%%%%%%%%%%%%%%%%%%%%%%%%%%%%%%%%%%%%%%%%%%
\subsection{Directional Curvature and Finite-Support Geometry}
\label{sec:curvature}

The matched Muon-S experiment in Section \ref{sec:spatial-block-ablation} shows that organizing independently polarized channel blocks by translation frequency improves the training trajectory relative to spatial-offset blocks. We next examine how this operator-aligned geometry acts locally within a trained network and why the critical Fourier grid is important for realizing it with finite-support kernels. We study shared states from a Muon-C U-Net trajectory at 40k, 200k, and 400k iterations. At each state, we use one fixed 512-example evaluation objective, construct the three exact-polar directions from the same Nesterov momentum signal, and jointly perturb 49 eligible $3\times3$ convolution kernels. Each direction is scaled per layer to the same convolution-operator-norm budget, estimated on a $32\times32$ Fourier grid, and a common radius multiplier $\alpha$ controls the perturbation size. Appendix~\ref{app:local-geometry-protocol} specifies the budgets and evaluation protocol.

To separate the linear benefit of a direction from its local curvature, we use the quadratic approximation
\begin{equation} \label{eqn:quad-approx}
L(\theta+\alpha V)-L(\theta) \approx -\alpha a+\frac{1}{2}\alpha^2 c_H, \qquad
a=-\langle \nabla L(\theta),V\rangle, \qquad
c_H=V^\top\nabla^2L(\theta)V.
\end{equation}
We estimate $c_H$ and the MSE Gauss-Newton curvature by central finite differences. The resulting local-geometry diagnostics show that Muon-C retains 73\%-89\% of unfolded Muon's linearized decrease while reducing the Hessian and Gauss-Newton curvature estimates to 45\%-60\% and 47\%-51\%, respectively. Its resulting quadratic-model maximum decrease, $a^2/(2c_H)$, is 12\%-38\% larger. Muon-S's Gauss-Newton curvature remains within 4\% of unfolded Muon's, whereas Muon-C's is approximately half of Muon-S's.

\begin{figure}[t!]
\centering
\includegraphics[width=0.9\linewidth,height=2.25in]{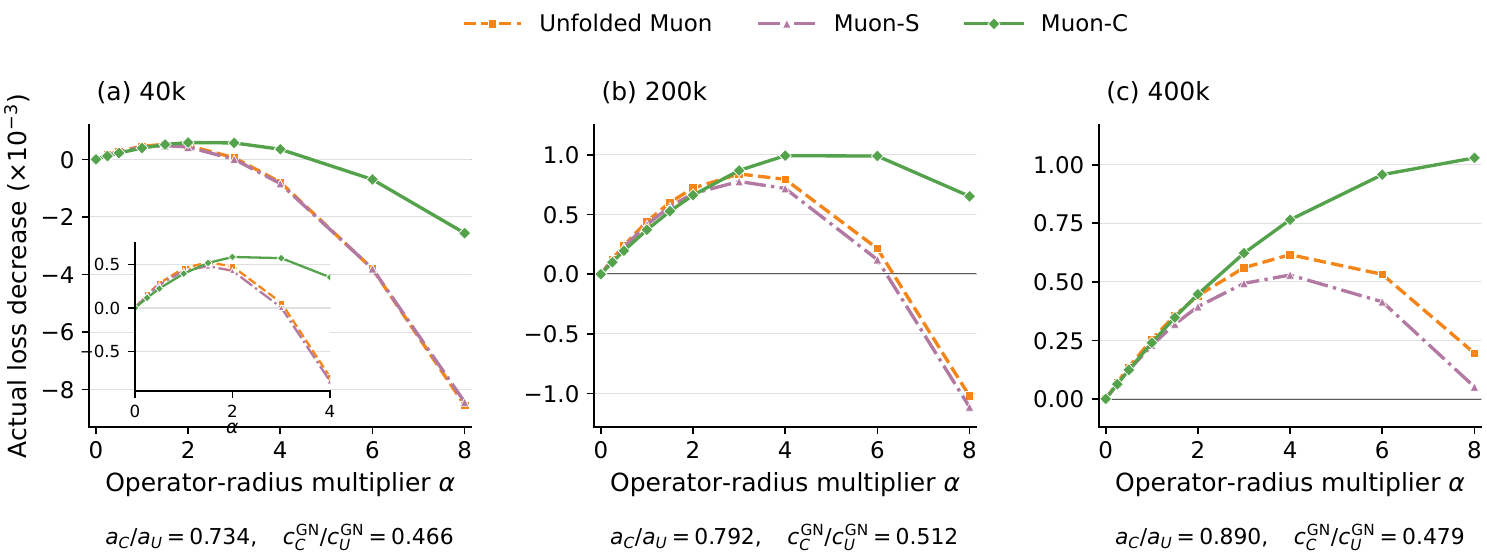}
\caption{Finite-step loss decrease under matched convolution-operator-norm budgets at 40k, 200k, and 400k iterations. Annotations compare Muon-C with unfolded Muon; panels use separate vertical scales. More negative is better.}
\label{fig:alignment-curvature}
\vspace{-0.2in}
\end{figure}

Figure \ref{fig:alignment-curvature} shows the corresponding finite-step behavior under the matched operator-norm budget. Unfolded Muon gives the largest decrease at $\alpha=1$ at all three states, consistent with its larger linearized decrease. Muon-C, however, tolerates a substantially larger useful step. Its best tested decrease occurs at $\alpha=2,4,8$ at 40k, 200k, and 400k iterations, compared with $1.5,3,4$ for both controls, and it improves the best tested decrease over unfolded Muon by 12\%, 18\%, and 67\%, respectively. Thus, the frequency-organized direction sacrifices some first-order decrease but substantially reduces directional curvature, allowing a larger operator-norm step and a greater finite-step improvement.

\begin{table}[b!]
\centering
\small
\resizebox{\linewidth}{!}{
\begin{tabular}{ccccc}
\toprule
Checkpoint & Layers & Supported discrepancy $d_{\mathrm{supp}}$ & Outside-support energy $\ell$ & Full discrepancy $d_{\mathrm{full}}$ \\
\midrule
40k  & 9 & 0.398 / 0.621 / 0.869 & 0.113 / 0.256 / 0.348 & 0.503 / 0.735 / 0.914 \\
200k & 9 & 0.426 / 0.623 / 0.850 & 0.130 / 0.252 / 0.342 & 0.536 / 0.729 / 0.900 \\
400k & 9 & 0.454 / 0.630 / 0.811 & 0.161 / 0.251 / 0.339 & 0.580 / 0.731 / 0.880 \\
\bottomrule
\end{tabular}}
\par
\normalsize
\caption{Full-grid finite-support obstruction across nine representative stride-one layers. Entries report minimum / median / maximum at each checkpoint for supported discrepancy $d_{\mathrm{supp}}$, outside-support energy $\ell$, and full discrepancy $d_{\mathrm{full}}$.}
\label{tab:support-diagnostic}
\end{table}

Table \ref{tab:support-diagnostic} examines a complementary question: whether the same frequency-wise polarization could instead be carried out on the full feature-frequency grid and then restricted back to the stored kernel support. For nine representative stride-one layers spanning $4\times4$ to $32\times32$ feature grids, the full-grid update places substantial energy outside the $3\times3$ stored support at every checkpoint. The median outside-support energy remains approximately 25\%, while the median discrepancy between the critical-grid update and the cropped full-grid update is 0.62--0.63; the corresponding full-update discrepancy is 0.73--0.74. These effects are stable from 40k through 400k iterations. Exact-SVD checks confirm that the full-grid oracle itself is solved accurately. Critical sampling makes independent frequency-wise polarization compatible with the original finite kernel support.

Taken together, Figure \ref{fig:alignment-curvature} and Table \ref{tab:support-diagnostic} clarify two complementary roles of Muon-C's operator-aligned construction. Frequency organization changes the local optimization geometry, producing substantially lower directional curvature and permitting larger useful operator-norm steps, while critical sampling makes this geometry realizable without violating finite kernel support. These diagnostics provide a network-level mechanism consistent with the stronger early training and FID trajectories observed in Sections \ref{sec:applied-update-rms} to \ref{sec:spatial-block-ablation}.

%%%%%%%%%%%%%%%%%%%%%%%%%%%%%%%%%%%%%%%%%%%%%%%%%%%%%%%%%%%%%%%%%%%%%%%%%%%%%%%%%%%%%%%
\section{Discussion}
\label{sec:discussion}

Muon is inherently representation dependent because its spectral-norm geometry is determined by the matrices, and more generally the independently constrained blocks, on which the polar update acts. The central lesson of this work is therefore broader than one implementation for convolution. Extending Muon to a structured parameter requires choosing not only a coordinate representation but also the linear maps that define its independently polarized blocks. A unitary change of coordinates alone cannot alter the global Muon direction, while a genuinely different geometry arises from changing the block partition. For convolution, translation equivariance provides a principled choice through the frequency-wise channel-transfer maps. Muon-C realizes this operator-aligned geometry on a critical Fourier grid, allowing these maps to be polarized independently while returning the update exactly to the original finite kernel support.

The theoretical and empirical results provide complementary support for this construction. Muon-C is the exact linear minimization oracle for the critically sampled convolution norm, and the critical DFT gives a unitary, bijective, and support-preserving representation of the stored coefficients. Relative to the continuous convolution-operator norm, its worst-case guarantee is no weaker than global unfolding for any finite kernel size and is strictly stronger for $3\times3$ kernels. The experiments show that this geometric change translates into improved optimization efficiency. Muon-C produces stronger early trajectories and compute-to-quality performance under controlled update scale, and the advantage persists across stochastic replication, optimizer-specific tuning, larger data scale, discriminative training, and architectures with widely different routing coverage. The matched Muon-S comparison supports the contribution of translation-frequency organization to the early trajectory gains.

Several questions remain open. On the theoretical side, establishing conditions under which improved operator-oracle alignment translates into faster convergence remains an important next step. The reduced directional curvature observed in Section \ref{sec:curvature} suggests one possible mechanism: operator-aligned first-order geometry may interact with local loss curvature to permit larger useful steps. Analyzing the finite-iteration Newton-Schulz approximation and extending the operator-norm characterization to strided and dilated convolution are additional directions for future work. On the empirical side, Muon-C introduces additional optimizer-side computation despite its resolution-independent critical-grid construction. The strongest empirical evidence is also concentrated in relatively small-resolution generative models and ImageNet-100 classification, rather than full-scale modern convolutional training. Broader wall-clock evaluations, together with extensions to full-resolution ImageNet-1k, larger kernels, and additional convolution-heavy tasks, would further test its practical scope. More broadly, the operator-aligned viewpoint suggests a route for extending Muon to other structured parameters by identifying symmetry-induced linear maps and defining optimizer geometry on those maps.

%%%%%%%%%%%%%%%%%%%%%%%%%%%%%%%%%%%%%%%%%%%%%%%%%%%%%%%%%%%%%%%%%%%%%%%%%%%%%%%%%%%%%%%
\appendix

%%%%%%%%%%%%%%%%%%%%%%%%%%%%%%%%%%%%%%%%%%%%%%%%%%%%%%%%%%%%%%%%%%%%%%%%%%%%%%%%%%%%%%%
\section{Proofs and Operator-Norm Analysis of Muon-C}
\label{app:theory}

All frequency-domain pairings use $\inner{A}{B}_{F}=\operatorname{tr}(A^*B)$ with the real part understood. For real spatial tensors, conjugate symmetry is preserved because $\operatorname{polar}(\overline A)=\overline{\operatorname{polar}(A)}$. Paired choices therefore invert to real tensors. At zero or rank-deficient blocks, we use the canonical partial polar factor and omit arbitrary null-space components.

%%%%%%%%%%%%%%%%%%%%%%%
\begin{proof}(Proposition~\ref{prop:polar-unitary-equivariance})
If $X=P\Sigma Q^*$ is a compact singular value decomposition, then $LXR=(LP)\Sigma(R^*Q)^*$ is also one, and hence $\operatorname{polar}(LXR)=(LP)(R^*Q)^*=L(PQ^*)R$. The zero case follows from $\operatorname{polar}(0)=0$. This completes the proof of Proposition~\ref{prop:polar-unitary-equivariance}.
\end{proof}

%%%%%%%%%%%%%%%%%%%%%%%
\begin{proof}(Proposition~\ref{prop:exact-lmo})
Equation~\eqref{eq:sample-kfb-relation} makes the constraint equivalent to $\opnorm{\widehat U(p,q)}\leq\rho/\sqrt n$ at every critical frequency. Parseval's identity gives
\[
\inner{M}{U}=\sum_{p,q}\mathrm{Re}\,\inner{\widehat M(p,q)}{\widehat U(p,q)}_{F},
\]
so the objective and constraint separate by block. Spectral--nuclear norm duality yields $\widehat U_{\mathrm f}^\star(p,q)=-(\rho/\sqrt n)\operatorname{polar}(\widehat M(p,q))$. For a real $M$, conjugate symmetry and polar conjugation make the inverse DFT real. This completes the proof of Proposition~\ref{prop:exact-lmo}.
\end{proof}

%%%%%%%%%%%%%%%%%%%%%%%
\emph{Full-grid circular linear minimization oracle.} The idealized full-grid reference allows one free coefficient matrix at every location of an $H\times W$ circular grid. Let $G,U\in\R^{C_{\mathrm{out}}\times C_{\mathrm{in}}\times H\times W}$, let $\Omega=[H]\times[W]$, and let $\widehat G=\cF_\Omega(G)$ and $\widehat U=\cF_\Omega(U)$ denote their orthonormal spatial DFTs. Define
\[
\|U\|_{\mathrm{block},\Omega} :=\max_{\omega\in\Omega}\opnorm{\widehat U(\omega)}.
\]
This norm is proportional to the induced $\ell_2$ norm of the corresponding circular-convolution operator. The proportionality constant rescales the trust-region radius but does not change the steepest direction.

\begin{proposition}[Idealized full-grid Fourier-block LMO]
\label{prop:block-lmo}
For full-grid momentum $G$, one solution of
\[
\min_{\|U\|_{\mathrm{block},\Omega}\leq\rho}\inner{G}{U}
\]
is specified by
\[
\widehat U^\star(\omega)
=-\rho\,\operatorname{polar}\bigl(\widehat G(\omega)\bigr),
\qquad \omega\in\Omega.
\]
\end{proposition}
\begin{proof}
Parseval's identity separates the objective as
\[
\inner{G}{U}=\sum_{\omega\in\Omega}\mathrm{Re}\,\inner{\widehat G(\omega)}{\widehat U(\omega)}_F.
\]
The constraint bounds each $\opnorm{\widehat U(\omega)}$ by $\rho$, so spectral--nuclear norm duality gives $\widehat U^\star(\omega)=-\rho\operatorname{polar}(\widehat G(\omega))$ independently at every frequency. Conjugate-symmetric choices invert to a real spatial tensor. This completes the proof of Proposition \ref{prop:block-lmo}.
\end{proof}

%%%%%%%%%%%%%%%%%%%%%%%
\begin{proof}(Theorem~\ref{thm:unified-norm-comparison})
Write $j=(p,q)$ and $A_j=A_U(\theta_p,\phi_q)$. Since
\[
B_UB_U^*=\sum_{u,v}U_{uv}U_{uv}^*\succeq U_{uv}U_{uv}^*,
\]
each $\opnorm{U_{uv}}\leq\patchnorm{U}$, proving $\spnorm{U}\leq\patchnorm{U}$. Discrete Fourier orthogonality gives
\begin{equation*}
\frac1n\sum_jA_jA_j^*=\sum_{u,v}U_{uv}U_{uv}^*=B_UB_U^*.
\end{equation*}
Therefore
\[
\patchnorm{U}^2
\leq\frac1n\sum_j\lambda_{\max}(A_jA_j^*)
\leq\samplenorm{U}^2.
\]
The critical frequencies lie in the continuous domain. Hence $\samplenorm{U}\leq\convnorm{U}$, completing \eqref{eq:norm-hierarchy}.

For the spatial upper bound, the triangle inequality gives
\[
\opnorm{A_U(\theta)}\leq\sum_{u,v}\opnorm{U_{uv}}\leq n\,\spnorm{U}.
\]
Let $Z_\theta$ vertically stack $e^{-\mathrm{i}\langle(u,v),\theta\rangle}I$ over $(u,v)\in\Omega_k$. Then $A_U(\theta)=B_UZ_\theta$ and $Z_\theta^*Z_\theta=nI$, so $\opnorm{A_U(\theta)}\leq\sqrt n\,\patchnorm{U}$.

Cardinal interpolation gives
\[
A_U(\theta_1,\theta_2)=\sum_{p,q}\ell_p^{(k_h)}(\theta_1)\ell_q^{(k_w)}(\theta_2)A_j.
\]
Taking operator norms and then the supremum yields $\convnorm{U}\leq\Lambda_{k_h}\Lambda_{k_w}\samplenorm{U}$, proving \eqref{eq:unified-upper-bounds}. Fourier orthogonality also gives $\sum_p|\ell_p^{(k)}(\theta)|^2=1$. Cauchy--Schwarz therefore gives $\Lambda_k\leq\sqrt k$, proving \eqref{eq:lebesgue-sqrt-bound}. For $k=3$,
\[
|\ell_p^{(3)}(\theta)|=\frac13\left|1+2\cos\!\left(\theta-\frac{2\pi p}{3}\right)\right|.
\]
The three expressions inside the absolute values sum to $3$, at most one is negative, and each is at least $-1$. Their absolute values therefore sum to at most $5$, with equality at $\theta=\pi$, so $\Lambda_3=5/3$. This completes the proof of Theorem~\ref{thm:unified-norm-comparison}.
\end{proof}

%%%%%%%%%%%%%%%%%%%%%%%
\begin{proof}(Corollary~\ref{cor:unified-lmo})
Spectral--nuclear norm duality gives the global and spatial oracles, and Proposition~\ref{prop:exact-lmo} gives the Fourier oracle. Let $\|\cdot\|_r$ be the corresponding surrogate norm and define
\[
\Delta_r=\max_{\|U\|_r\leq\rho}-\inner{M}{U},
\qquad
\Delta_{\mathrm{conv}}=\max_{\convnorm{U}\leq\rho}-\inner{M}{U}.
\]
The theorem gives $\|U\|_r\leq\convnorm{U}\leq\kappa_r\|U\|_r$. Thus the continuous ball is contained in the surrogate ball and $\Delta_r\geq\Delta_{\mathrm{conv}}$, while $\convnorm{U_r^\star/\kappa_r}\leq\|U_r^\star\|_r\leq\rho$. Consequently,
\[
-\inner{M}{U_r^\star/\kappa_r}
=\frac{\Delta_r}{\kappa_r}
\geq\frac{\Delta_{\mathrm{conv}}}{\kappa_r}, 
\]
which proves  \eqref{eq:unified-lmo-factor}. This completes the proof of Corollary~\ref{cor:unified-lmo}.
\end{proof}

%%%%%%%%%%%%%%%%%%%%%%%%%%%%%%%%%%%%%%%%%%%%%%%%%%%%%%%%%%%%%%%%%%%%%%%%%%%%%%%%%%%%%%%
\section{Implementation and Experimental Details}
\label{app:implementation}

%%%%%%%%%%%%%%%%%%%%%%%
\emph{Muon-C implementation and routing.} Muon-C applies a two-dimensional orthonormal FFT on the kernel-sized spatial grid and uses five float32 complex Newton--Schulz iterations to approximate the hard-polar direction. We use momentum 0.95 with Nesterov momentum and no support-constraint iteration. Conv2d kernels with spatial area $k_hk_w>1$ are routed through Muon-C, while $1\times1$ convolutions, biases, normalization parameters, and residual matrices use AdamW. Parameters with matching shapes and settings are batched. The global unfolded-Muon baseline uses the same routing boundary but applies matrix Muon to each eligible kernel's mode-0 unfolding.
For a convolution with $g$ groups, Muon-C independently polarizes each group's $(C_{\mathrm{out}}/g)\times(C_{\mathrm{in}}/g)$ transfer matrix at each frequency. Thus, ConvNeXtV2-T's depthwise convolutions use one scalar block per channel and frequency, without mixing channels.

For each frequency block $B$, let $\widetilde B=B^*$ if $B$ has more rows than columns, and $\widetilde B=B$ otherwise. The complex Newton--Schulz routine uses
\[
X_0=\frac{\widetilde B}{\|\widetilde B\|_F+\epsilon},
\qquad A_j=X_jX_j^*,
\qquad X_{j+1}=aX_j+(bA_j+cA_j^2)X_j,
\]
with $(a,b,c)=(3.4445,-4.775,2.0315)$, $\epsilon=10^{-8}$, and $j=0,\ldots,4$. It returns $X_5^*$ for a transposed input and $X_5$ otherwise. Computation uses complex64; the same $\epsilon$ stabilizes the update-RMS normalization in Algorithm~\ref{alg:muon-c-step}.

%%%%%%%%%%%%%%%%%%%%%%%
\emph{Training configurations.} The CIFAR-10 U-Net uses $32\times32$ images, batch size 128, mixed precision, 400k training steps, a 5k-step warmup followed by a constant learning rate, gradient clipping at 1.0, and EMA decay 0.9999. ImageNet-1k-32 uses the same U-Net architecture, batch size, precision, warmup-constant schedule, 400k-step budget, EMA protocol, data pipeline, and parameter-routing rule. ImageNet-100 uses $224\times224$ inputs and trains ResNet-34, ResNet-50, and ConvNeXtV2-T for 100 epochs, reporting full-validation top-1 accuracy under the same routing and applied-update-scale protocol.
All ImageNet-100 runs use batch size 256 and a 1k-step linear warmup followed by cosine decay to zero. The base learning rate and weight decay are $(2\times10^{-4},0)$ for both ResNets and $(2\times10^{-3},0.05)$ for ConvNeXtV2-T. Training uses random resized crops with area scale $[0.08,1]$ and horizontal flips with probability 0.5. Validation resizes the shorter side to 256 pixels and center-crops to 224 pixels. Both splits use ImageNet channel normalization. Figure~\ref{fig:imagenet100-classification-convergence} shows one training run per method and architecture.
Table~\ref{tab:transfer-results} reports the mean and standard deviation of three independent runs with seeds 42--44.

%%%%%%%%%%%%%%%%%%%%%%%
\emph{Learning-rate selection and replication.} Common-learning-rate flow-matching runs use a base learning rate of $2\times10^{-4}$. CIFAR-10 repetitions use paired training seeds 42--44. Each equal-budget sweep evaluates ten learning rates from $10^{-4}$ to $10^{-3}$ in increments of $10^{-4}$ and selects the value with the lowest final FID-50k. The selected learning rates for Adam, global unfolded Muon, and Muon-C are $(2,8,5)\times10^{-4}$ on CIFAR-10 and $(3,10,6)\times10^{-4}$ on ImageNet-1k-32. Weight decay is zero in the CIFAR experiments, so AdamW on the hybrid fallback route is equivalent to the pure Adam baseline there.

% %%%%%%%%%%%%%%%%%%%%%%%
% \paragraph{Applied-update RMS calibration.}
% To isolate optimizer geometry from update magnitude in the primary CIFAR-10 comparison, we calibrate Muon-C to the applied parameter-space RMS of global unfolded Muon on the eligible spatial-convolution route $\mathcal{P}$. We define the actual post-step displacement by
% \[
% \operatorname{RMS}(\Delta\mathcal{P}) = \left( \frac{\sum_{W\in\mathcal{P}}\|\Delta W\|_F^2} {\sum_{W\in\mathcal{P}}|W|} \right)^{1/2}.
% \]
% A five-step controlled probe shares initialization, mini-batches, and stochastic draws across methods. The fixed calibration is determined from steps 2--5 after the first warmup-scale step. At learning rate $2\times10^{-4}$, the probe measures uncalibrated route RMS of $5.56\times10^{-6}$ for Muon-C and $1.62\times10^{-6}$ for global unfolded Muon. The matched-update runs therefore apply one fixed multiplier of 0.291 to the Muon-C route. After calibration, Muon-C has mean route RMS $1.42\times10^{-6}$ compared with $1.60\times10^{-6}$ for global unfolded Muon, a ratio of 0.888. Thus the controlled comparison is conservative with respect to update magnitude, with Muon-C taking the smaller applied displacement. Table \ref{tab:applied-rms-audit} gives the step-level and layer-level audit.

%%%%%%%%%%%%%%%%%%%%%%%
\emph{Evaluation and model-compute accounting.} Every FID-50k evaluation uses 50k generated samples, 100 Dopri5 sampling steps, fixed real-data statistics, and EMA weights. The shared checkpoints are 20k, 40k, 80k, 160k, 200k, 300k, and 400k iterations. The Muon-S control uses the same evaluation protocol. Figure \ref{fig:matched-rms-fid-flops} reports cumulative U-Net forward--backward model FLOPs, using $4.78\times10^{12}$ FLOPs per batch; optimizer, evaluation, and data-pipeline computation are excluded. Runtime measurements in Table~\ref{tab:cifar-runtime-efficiency} use one NVIDIA RTX PRO 6000 Blackwell Max-Q Workstation Edition GPU. Each optimizer runs in a separate process with 50 warmup steps followed by 500 measured steps at batch size 128. We report mean step time, with CUDA synchronization at both timing boundaries, including data loading and transfer, forward/backward computation, gradient clipping, optimizer and scheduler steps, gradient clearing, and EMA updates. Peak-memory statistics are reset after warmup; we report maximum allocated memory over the timed window. Evaluation is excluded.

%%%%%%%%%%%%%%%%%%%%%%%
\emph{Computational and memory cost.} A real-input FFT on an $a\times b$ grid stores $a(\lfloor b/2\rfloor+1)$ complex channel-transfer blocks. For $C_{\mathrm{in}}=C_{\mathrm{out}}=256$, a $3\times3$ critical grid has six blocks and occupies approximately 3 MiB in complex64. Feature grids of size 8, 16, and 32 use 40, 144, and 544 blocks, respectively; the $32\times32$ tensor occupies approximately 272 MiB and the $224\times224$ tensor approximately 12.4 GiB, before Newton-Schulz temporaries. These values are storage estimates for the frequency-block tensors rather than end-to-end training memory; measured peak allocated memory is reported in Table \ref{tab:cifar-runtime-efficiency}.

%%%%%%%%%%%%%%%%%%%%%%%%%%%%%%%%%%%%%%%%%%%%%%%%%%%%%%%%%%%%%%%%%%%%%%%%%%%%%%%%%%%%%%%
\section{Additional Empirical Results and Diagnostics}
\label{app:empirical-diagnostics}

%%%%%%%%%%%%%%%%%%%%%%%%%%%%%%%%%%%%%%%%%%%%%%%%%%%%%%%%%%%%%%%%%%%%%%%%%%%%%%%%%%%%%%%
\subsection{Applied-Update RMS Diagnostics}

Table~\ref{tab:applied-rms-audit} reports the actual parameter displacements under the AdamW-referenced RMS scaling described in Section~\ref{sec:practical}. The probe uses shared initialization, mini-batches, and stochastic draws and measures post-step RMS on the eligible spatial-Conv2d route.

\begin{table}[b!]
\centering
\small
\begin{tabular}{lccc}
\toprule
Measurement & \globalmuon{} & \method{} & Ratio \\
\midrule
\multicolumn{4}{l}{\textit{Global route RMS ($\times 10^{-6}$)}} \\
Step 2 & 1.301 & 1.214 & 0.934 \\
Step 3 & 1.599 & 1.418 & 0.887 \\
Step 4 & 1.706 & 1.494 & 0.875 \\
Step 5 & 1.779 & 1.547 & 0.869 \\
Mean   & 1.596 & 1.418 & 0.888 \\
\midrule
\multicolumn{4}{l}{\textit{Layerwise RMS-ratio summary over 52 layers}} \\
Mean of layer ratios   & \multicolumn{2}{c}{---} & 0.733 \\
Median layer ratio     & \multicolumn{2}{c}{---} & 0.744 \\
Range                  & \multicolumn{2}{c}{---} & [0.137, 1.100] \\
Layers with ratio $\leq 1$ & \multicolumn{2}{c}{---} & 43/52 \\
\bottomrule
\end{tabular}
\par
\normalsize
\caption{Applied-update RMS on the eligible spatial-Conv2d route. Ratios are Muon-C divided by global unfolded Muon; layer summaries average available nonzero ratios over steps 2 to 5.}
\label{tab:applied-rms-audit}
\end{table}

Muon-C has a smaller global route RMS at every measured step, with a mean ratio of 0.888 relative to global unfolded Muon. The same pattern holds across layers: 43 of 52 layers have mean RMS ratios no greater than one, and the median layerwise ratio is 0.744.

%%%%%%%%%%%%%%%%%%%%%%%%%%%%%%%%%%%%%%%%%%%%%%%%%%%%%%%%%%%%%%%%%%%%%%%%%%%%%%%%%%%%%%%
\subsection{Robustness across Seeds and Training Budgets}
\label{app:sweep-budget}

%%%%%%%%%%%%%%%%%%%%%%%
\subsubsection{Replication across Training Seeds}
Table~\ref{tab:cifar-common-lr-seeds} resolves the common-learning-rate CIFAR-10 aggregate in Table~\ref{tab:common-lr-results} into three paired training runs. The training seed is shared across methods within each column, and all runs use base learning rate $2\times10^{-4}$ and the same final FID-50k protocol.

\begin{table}[t!]
\centering
\small
\begin{tabular}{lcccc}
\toprule
Optimizer & Seed 42 & Seed 43 & Seed 44 & Mean $\pm$ std. \\
\midrule
Adam & 3.999 & 3.953 & 3.869 & $3.941\pm0.066$ \\
\globalmuon{} & 3.908 & 3.894 & 3.761 & $3.854\pm0.081$ \\
\method{} & \textbf{3.588} & \textbf{3.540} & \textbf{3.681} & $\mathbf{3.603\pm0.072}$ \\
\bottomrule
\end{tabular}
\par
\normalsize
\caption{CIFAR-10 final FID-50k across three paired common-learning-rate training seeds. The final column reports the mean and sample standard deviation. Lower is better.}
\label{tab:cifar-common-lr-seeds}
\end{table}

Muon-C achieves the best endpoint in every paired seed. Averaged across the three runs, it improves FID by 0.252 over global unfolded Muon and by 0.338 over Adam, with its seedwise gain over global unfolded Muon ranging from 0.080 to 0.354.

%%%%%%%%%%%%%%%%%%%%%%%
\subsubsection{Performance across Training Budgets}

\begin{table}[t!]
\centering
\resizebox{0.9\linewidth}{!}{
\begin{tabular}{lcccc}
\toprule
Optimizer & Selected LR & FID at $0.1B$ $\downarrow$ & FID at $0.5B$ $\downarrow$ & FID at $B$ $\downarrow$ \\
\midrule
\multicolumn{5}{l}{\textit{CIFAR-10}} \\
Adam & $2\times10^{-4}$ & $51.37\pm0.25$ & $4.41\pm0.04$ & $3.83\pm0.03$ \\
Global unfolded Muon & $8\times10^{-4}$ & $24.97\pm0.05$ & $3.89\pm0.07$ & $3.54\pm0.04$ \\
\method{} & $5\times10^{-4}$ & $\mathbf{10.37\pm0.06}$ & $\mathbf{3.85\pm0.06}$ & $\mathbf{3.42\pm0.03}$ \\
\addlinespace
\multicolumn{5}{l}{\textit{ImageNet-1k-32}} \\
\adamw{} & $3\times10^{-4}$ & $58.87\pm0.06$ & $15.11\pm0.03$ & $13.53\pm0.05$ \\
Global unfolded Muon & $1\times10^{-3}$ & $28.02\pm0.06$ & $14.10\pm0.04$ & $13.24\pm0.04$ \\
\method{} & $6\times10^{-4}$ & $\mathbf{21.14\pm0.05}$ & $\mathbf{13.96\pm0.05}$ & $\mathbf{12.02\pm0.02}$ \\
\bottomrule
\end{tabular}}
\par
\normalsize
\caption{Sweep-selected flow-matching performance at fixed fractions of the 400k-step training budget. The selected learning rate is fixed across budget fractions. Lower is better.}
\label{tab:main-results-full}
\end{table}

Table~\ref{tab:main-results-full} evaluates the sweep-selected runs at three fractions of the budget $B=400$k. For each optimizer, the learning rate selected by final FID is held fixed when evaluating the earlier budget fractions. Muon-C has the lowest FID at every reported fraction on both datasets. Its margin is largest at $0.1B$, where it reaches 10.37 on CIFAR-10 compared with 24.97 for global unfolded Muon and 51.37 for Adam, and 21.14 on ImageNet-1k-32 compared with 28.02 and 58.87, respectively.

These results complement the final tuned comparison in Table \ref{tab:main-results}. The selected runs retain the same optimizer ordering at $0.1B$ and $0.5B$.

%%%%%%%%%%%%%%%%%%%%%%%%%%%%%%%%%%%%%%%%%%%%%%%%%%%%%%%%%%%%%%%%%%%%%%%%%%%%%%%%%%%%%%%
\subsection{Local-Geometry and Finite-Support Diagnostics}

%%%%%%%%%%%%%%%%%%%%%%%
\subsubsection{Local-Geometry Protocol}
\label{app:local-geometry-protocol}
This diagnostic provides the detailed protocol for Figure \ref{fig:alignment-curvature} in Section \ref{sec:curvature}. We examine shared states from a Muon-C CIFAR-10 U-Net trajectory at 40k, 200k, and 400k iterations. At each state, we use one fixed 512-example evaluation objective and construct the global unfolded-Muon, Muon-S, and Muon-C directions from the same Nesterov momentum signal. For this diagnostic, all three directions use exact SVD-based polar factors rather than the Newton-Schulz approximation used during training. We jointly perturb the 49 eligible $3\times3$ stride-one convolution kernels shared by the comparison.

Each candidate direction is scaled per layer to the same convolution-operator-norm budget, estimated on a $32\times32$ Fourier grid. Let $D_{\ell,\mathrm{unf}}$ denote the raw unfolded-Muon direction for layer $\ell$, and let $\|\cdot\|_{\mathrm{conv},32}$ denote this grid estimate. We set
\[
\rho_\ell=4\times10^{-5}\,
\frac{\|D_{\ell,\mathrm{unf}}\|_{\mathrm{conv},32}}
{\operatorname{RMS}(D_{\ell,\mathrm{unf}})},
\qquad
V_{\ell,r}=\rho_\ell\,
\frac{D_{\ell,r}}{\|D_{\ell,r}\|_{\mathrm{conv},32}},
\]
where $\operatorname{RMS}(D)=\|D\|_F/\sqrt{|D|}$ and $r$ indexes the three methods. Thus, the unfolded-Muon perturbation has coefficient RMS $4\times10^{-5}$ at unit radius. Figure~\ref{fig:alignment-curvature} evaluates the joint perturbation $\alpha V_r$ for $\alpha\in\{0,0.25,0.5,1,1.5,2,3,4,6,8\}$.

For a candidate direction $V$, Section \ref{sec:curvature} uses the local quadratic approximation \eqref{eqn:quad-approx}. We use central finite differences at $\theta\pm hV$ with $h=0.25$ in the radius-multiplier units above. The Hessian estimate is
\[
\widehat c_H=\frac{L(\theta+hV)-2L(\theta)+L(\theta-hV)}{h^2}.
\]
The MSE Gauss--Newton estimate is twice the mean squared central difference of model outputs, with each output difference divided by $2h$. Both estimates use the same fixed inputs and stochastic draws at the perturbed states.

%%%%%%%%%%%%%%%%%%%%%%%
\subsubsection{Finite-Support Diagnostic}
Table \ref{tab:support-diagnostic} in Section \ref{sec:curvature} examines whether independent frequency-wise polarization on the full feature-frequency grid can be restricted back to the stored finite kernel support. Using stored Muon-C momenta from a completed CIFAR-10 run, let $U_k$ denote the update obtained by polarizing on the $3\times3$ critical grid and let $U_\Omega$ denote the update obtained from the same momentum on the layer's full feature grid. Let $P$ crop a full-grid update to the stored support and let $E$ zero-embed a cropped kernel back into the feature grid. We measure
\[
d_{\mathrm{supp}}
=
\frac{\|U_k-PU_\Omega\|_F}{\|PU_\Omega\|_F},
\qquad
\ell
=
\frac{\|U_\Omega-EPU_\Omega\|_F^2}{\|U_\Omega\|_F^2},
\qquad
d_{\mathrm{full}}
=
\frac{\|EU_k-U_\Omega\|_F}{\|U_\Omega\|_F}.
\]
Here, $d_{\mathrm{supp}}$ measures the discrepancy between the critical-grid update and the coefficients retained after cropping the full-grid update, $\ell$ is the fraction of full-grid energy outside the stored support, and $d_{\mathrm{full}}$ measures their discrepancy on the complete feature grid.

The diagnostic uses nine representative eligible stride-one layers spanning $4\times4$ to $32\times32$ feature grids at 40k, 200k, and 400k iterations. As reported in Table \ref{tab:support-diagnostic}, the obstruction is substantial and stable throughout training. As numerical checks, the finite-support identity
\[
d_{\mathrm{full}}^2=(1-\ell)d_{\mathrm{supp}}^2+\ell
\]
has maximum residual $1.79\times10^{-7}$. Exact-SVD checks of the idealized full-grid LMO on one layer at each checkpoint attain maximum relative objective gap $9.27\times10^{-8}$ and feasibility error $7.15\times10^{-7}$. Thus, the full-grid oracle is solved accurately; the observed discrepancy reflects the incompatibility between independent full-grid frequency polarization and the original finite kernel support.

%%%%%%%%%%%%%%%%%%%%%%%%%%%%%%%%%%%%%%%%%%%%%%%%%%%%%%%%%%%%%%%%%%%%%%%%%%%%%%%%%%%%%%%
\bibliography{ref-muon}

\end{document}